\documentclass{article}

\usepackage{microtype}
\usepackage{graphicx}
\usepackage{subcaption}
\usepackage{booktabs} 

\usepackage{natbib}

\usepackage{algorithm}
\usepackage{algorithmic}
\usepackage{xspace}
\usepackage{amsmath, amssymb,amsfonts,amsthm}

\newtheorem{theorem}{Theorem}
\newtheorem{lemma}{Lemma}
\newtheorem{cor}{Corollary}

\RequirePackage{enumitem}
\setlist{topsep=0.5em}

\usepackage{tikz}
\usepackage{pgfplots}
\usetikzlibrary[pgfplots.colormaps]
\newlength\figureheight
\newlength\figurewidth

\RequirePackage{mathtools}
\RequirePackage{bm}
\def\*#1{\bm{#1}}
\newcommand{\defeq}{\vcentcolon=}

\newcommand{\bigsmid}{\ \big\vert\ }

\newcommand{\bigsetdef}[2]{\big\{ #1 \bigsmid #2 \big\}}

\newcommand*\LET[2]{\STATE #1 $\gets$ #2}
\newcommand{\argmax}{\operatornamewithlimits{argmax}}

\newcommand{\LLC}{$L$-Lipschitz continuous}
\newcommand{\Reps}{R_{\epsilon}}
\newcommand{\Rbeps}{\bar{R}_{\epsilon}}

\newcommand{\Rbo}{\bar{R}_{0}}

\newcommand{\T}{T_t}
\newcommand{\algo}{{\sc\textsf{StageOpt}}\xspace}
\newcommand{\safeopt}{{\sc\textsf{SafeOpt}}\xspace}

\newcommand{\kersparring}{{\sc\textsf{KernelSelfSparring}}\xspace}

\usepackage{hyperref}

\usepackage[accepted]{icml2018}

\icmltitlerunning{Stagewise Safe Bayesian Optimization with Gaussian Processes}

\begin{document}

\twocolumn[
\icmltitle{Stagewise Safe Bayesian Optimization with Gaussian Processes}



\icmlsetsymbol{equal}{*}

\begin{icmlauthorlist}
\icmlauthor{Yanan Sui}{to}
\icmlauthor{Vincent Zhuang}{to}
\icmlauthor{Joel W. Burdick}{to}
\icmlauthor{Yisong Yue}{to}
\end{icmlauthorlist}

\icmlaffiliation{to}{California Institute of Technology, Pasadena, CA, USA}

\icmlcorrespondingauthor{Yanan Sui}{ysui@caltech.edu}
\icmlcorrespondingauthor{Vincent Zhuang}{vzhuang@caltech.edu}
\icmlcorrespondingauthor{Joel W. Burdick}{jwb@robotics.caltech.edu}
\icmlcorrespondingauthor{Yisong Yue}{yyue@caltech.edu}

\icmlkeywords{Safe Optimization, Bayesian Optimization, Gaussian Processes}

\vskip 0.3in
]



\printAffiliationsAndNotice{}  


\begin{abstract} 
Enforcing safety is a key aspect of many problems pertaining to sequential decision making under uncertainty, which require the decisions made at every step to be both informative of the optimal decision and also safe. For example, we value both efficacy and comfort in medical therapy, and efficiency and safety in robotic control. We consider this problem of optimizing an unknown utility function with absolute feedback or preference feedback subject to unknown safety constraints. We develop an efficient safe Bayesian optimization algorithm, \algo, that separates safe region expansion and utility function maximization into two distinct stages. Compared to existing approaches which interleave between expansion and optimization, we show that \algo is more efficient and naturally applicable to a broader class of problems. We provide theoretical guarantees for both the satisfaction of safety constraints as well as convergence to the optimal utility value. We evaluate \algo on both a variety of synthetic experiments, as well as in clinical practice.  We demonstrate that \algo is more effective than existing safe optimization approaches, and is able to safely and effectively optimize spinal cord stimulation therapy in our clinical experiments.
\end{abstract}
\section{Introduction}
\label{sec:introduction}

Bayesian optimization is a well-established approach for sequentially optimizing unknown utility functions. By leveraging regularity assumptions such as smoothness and continuity, such techniques offer efficient solutions for a wide range of high-dimensional problem settings such as experimental design and personalization in recommender systems. 


Many of these applications are also subject to a variety of safety constraints, so that actions cannot be freely chosen from the entire input space. For instance, in safe Bayesian optimization, any chosen action during optimization must be known to be ``safe'', regardless of the reward from the utility function. Typically, one is initially given a small region of the decision/action space that is known to be safe, and must iteratively expand the safe action region during optimization \citep{sui15icml}.


A motivating application of our work is a clinical setting, where physicians need to sequentially choose among a large set of therapies \citep{sui2017correlational}. The effectiveness and safety of different therapies are initially unknown, and can only be determined through sequential tests starting from some initial set of well-studied therapies. A natural way to explore is to start from some therapies similar to these initial ones, since their efficacy and safety would not differ too greatly. By iteratively repeating this process, one can gradually explore the utility and safety landscapes in a safe fashion.

\textbf{Our contributions.}
We propose a novel safe Bayesian optimization algorithm, \algo, to address the challenge of efficiently identifying the total safe region and optimizing the utility function within the safe region. In contrast to previous safe Bayesian optimization work \cite{sui15icml,berkenkamp16bayesian} which interleaves safe region expansion and optimization, \algo is a stagewise algorithm which first expands the safe region and then optimizes the utility function. \algo is well suited for settings in which the safety and utility functions are very different (e.g., temperature vs gripping force), i.e. lie on different scales or amplitudes. Furthermore, in settings in which the utility and safety functions are measured in different ways, it is natural to have a separate first stage dedicated to safe region expansion. For example, in clinical trials we may wish to spend the first stage only querying the patient about the comfort of the stimulus, as opposed to having to measure the utility and comfort simultaneously. 


Conceptually, \algo models the safety function(s) and utility function as sampled functions from different Gaussian processes (GPs), and uses confidence bounds to assess the safety of unexplored decisions. We provide theoretical results for \algo under the assumptions that (1) the safety and utility functions have bounded norms in their Reproducing Kernel Hilbert Spaces (RKHS) associated with the GPs, and (2) the safety functions are Lipschitz-continuous, which is guaranteed by many common kernels. We guarantee  (with high probability)  the convergence of \algo to the safely reachable optimum decision. In addition to simulation experiments, we apply \algo to a clinical setting of optimizing spinal cord stimulation for patients with spinal cord injuries. Compared to expert physicians, we find that \algo explores a larger safe region and finds better stimulation strategy.

\section{Related Work}
\label{sec:related}
Many Bayesian optimization methods often model the unknown underlying functions as Gaussian processes (GPs), which are smooth, flexible, nonparametric models \cite{rasmussen06}. GPs are widely used as a regularity assumption in many Bayesian optimization techniques, since they can easily encode prior knowledge and explicitly model variance.

The fundamental tradeoff between exploration and exploitation in sequential decision problems is commonly formalized as the multi-armed bandit problem (MAB), introduced by \citet{robbins52}. In MAB, each decision is associated with a stochastic reward with initially unknown distribution. The goal of a bandit algorithm is to maximize the cumulative reward. In a variant called ``best-arm identification'' \citep{audibert10}, one seeks to identify the decision with highest reward with minimal trials.  It has been widely studied under a variety of different situations (cf., \citet{bubeck2012regret} for an overview). Many efficient algorithms build on the methods of {\em upper confidence bounds} proposed in \citet{auer02}, and {\em Thompson sampling} proposed in \citet{thompson1933likelihood}. Their key ideas are to use posterior distributions of rewards to implicitly negotiate the explore-exploit tradeoff by optimistic sampling. This idea naturally extends to bandit problems with complex (or even infinite) decision sets under certain regularity conditions of the reward function \citep{dani08,Kleinberg08,Bubeck08}.

In the kernelized setting, several algorithms with theoretical guarantees have been proposed. \citet{srinivas10} propose the GP-UCB algorithm, which uses confidence bounds to address bandit problems with a reward function modeled using a Gaussian process. \citet{gotovos13} studies active sampling for localizing level sets, finding where the objective crosses a specified threshold.
\citet{chowdhury2017kernelized} extends the work of \citet{srinivas10} by proving tighter bounds as well as providing guarantees for a GP-based Thompson sampling algorithm. However, none of these algorithms are designed to work with safety constraints, and often violate them in practice \cite{sui15icml}. 
There are also algorithms without theoretical guarantees. \citet{gelbart2014bayesian} studies a constrained Expected Improvement algorithm for Bayesian optimization with unknown constraints. \citet{hernandez2016general} considers a general framework for constrained Bayesian optimization using information-based search. 


The problem of safe exploration has been considered in control and reinforcement learning \citep{hans08safe,gillula11guaranteed,garcia12safe,turchetta16safemdp}. These methods typically consider the problem of safe exploration in MDPs. They ensure safety by restricting policies to be ergodic with high probability and able to recover from any state visited. 
The safe optimization problem has also been studied under the restriction of the bandit/optimization setting, where decisions do not cause state transitions. This leads to simpler algorithms (\safeopt) with stronger guarantees \cite{sui15icml,berkenkamp16bayesian}, and fits well to safe sampling problems and applications. There are other safe algorithms \cite{schreiter2015safe,wu2016conservative} under different active learning settings. Our work builds upon the \safeopt approach, with stronger empirical performance and convergence rates on a broad class of safety functions.


\section{Problem Statement}
\label{sec:problem}

We consider a sequential decision problem in which we seek to optimize an unknown utility function $f:D\rightarrow\mathbb{R}$ from noisy evaluations at iteratively chosen sample points $\*x_1, \*x_2, \ldots, \*x_t, \ldots \in D$. However, we further require that each of these sample points are ``safe": that is, for each of $n$ unknown safety functions $g_i: D\rightarrow \mathbb{R}$ at $g_i(\*x_t)$ lies above some threshold $h_i\in\mathbb{R}$. We can formally write our optimization problem as follows:
\begin{equation}
\max_{\*x\in D} f(\*x) \quad\text{subject to}\enspace g_i(\*x) \geq h_i \text{ for } i = 1,\ldots, n
\end{equation}

\paragraph{Regularity assumptions.} In order to model the utility function and the safety functions, we use Gaussian processes (GPs), which are smooth yet flexible nonparametric models. Equivalently, we assume that $f$ and all $g_i$ have bounded norm in the associated Reproducing Kernel Hilbert Space (RKHS). A GP is fully specified by its mean function $\mu(\*x)$ and covariance function $k(\*x, \*x')$; in this work, we assume WLOG GP priors to have zero mean (i.e. $\mu(\*x) = 0)$. 
We further assume that each safety function $g_i$ is $L_i$-Lipschitz continuous with respect to some metric $d$ on $D$. This assumption is quite mild, and is automatically satisfied by many commonly-used kernels \cite{srinivas10,sui15icml}.


\paragraph{Feedback models.}
We primarily consider noise-perturbed feedback, in which our observations are perturbed by i.i.d.~Gaussian noise, i.e., for samples 
at points $A_T = [\*x_1 \dots \*x_T]^T\subseteq D$, we have $\*y_t = f(\*x_t) + n_t$ where $n_t \sim N(0, \sigma^2)$. The posterior over $f$ is then also Gaussian with mean $\mu_T(\*x)$, covariance $k_T(\*x, \*x')$ and variance $\sigma_T^2(\*x, \*x')$ that satisfy,
\begin{align*}
\mu_T(\*x) &= \*k_T(\*x)^T(\*K_T + \sigma^2\*I)^{-1}\*y_T\\
k_T(\*x, \*x') &= k(\*x, \*x') - \*k_T(\*x)^T(\*K_T + \sigma^2\*I)^{-1}\*k_T(\*x')\\
\sigma_T^2(\*x) &= k_T(\*x, \*x),
\end{align*}
where $\*k_T(\*x) = [k(\*x_1, \*x) \dots k(\*x_T, \*x)]^T$ and $\*K_T$ is the positive definite kernel matrix $[k(\*x, \*x')]_{\*x, \*x' \in A_T}$.

We also consider the case in which only preference feedback is available for the utility function. This setting is often used to characterize real-world applications that elicit subjective human feedback.  One way to formalize the online optimization problem is the dueling bandits problem \citep{yue2012k,sui2017multi}. In the basic dueling bandits formulation, given two points $\*x_1$ and $\*x_2$, we stochastically receive binary 0/1 feedback according to a Bernoulli distribution with parameter $\phi(f(\*x_1), f(\*x_2))$, where $\phi$ is a \emph{link function} mapping $\mathbb{R}\times\mathbb{R}$ to $[0, 1]$. For example, a common link function is the logit function $\phi(x,y) = (1+\exp(y-x))^{-1}$. 

To our knowledge, there are no existing algorithms for the safe Bayesian dueling bandit setting. Although our proposed algorithm is amenable to the full dueling bandits setting (as discussed later), to compare against existing algorithms, we consider the restricted dueling problem in which at timestep $t$ one receives preference feedback between $\*x_t$ and $\*x_{t-1}$. The pseudocode for our proposed algorithm under this type of dueling feedback can be found in Appendix \ref{app:duel}.

\paragraph{Safe optimization.}
Using a uniform zero-mean prior (as is typical in many Bayesian optimization approaches) does not provide sufficient information to identify any point as safe with high probability. Therefore, we additionally assume that we are given an initial ``seed'' set of safe decision(s), which we denote as $S_0\subset D$. Note that given an arbitrary seed set, it is not guaranteed that we will be able to discover the globally optimal decision $\*x^* = \argmax_{\*x \in D}{f(\*x)}$, e.g. if the safe region around $\*x^*$ is topologically separate from that of $S_0$. Instead, we can formally define the optimization goal for a given seed via the {\em one-step reachability} operator:
\[\begin{split}
\Reps(S) \!\defeq\! S \cup \bigcap_i\bigsetdef{\*x \in D}{\exists &\*x' \in S, \\
&g_i(\*x') - \epsilon - L_id(\*x', \*x) \geq h_i},
\end{split}\]
which gives the set all of points that can be established as safe given evaluations of $g_i$'s on $S$ with $\epsilon$ noise. Then, given some finite horizon $T$, we can define the subset of $D$ reachable after $T$ iterations from the initial safe seed set $S_0$ as the following:
$$\Reps^T(S_0) \defeq \underbrace{\Reps(\Reps\ldots(\Reps}_{T\ \mathrm{times}}(S_0))\ldots).$$ Thus, our optimization goal is $\argmax_{\*x\in\Reps^T(S_0)} f(\*x)$.

\section{Algorithm}
\label{sec:algorithm}

We now introduce our proposed algorithm, \algo, for the safe exploration for optimization problem.

\paragraph{Overview.} We start with a high-level description of \algo. \algo separates the safe optimization problem into two stages: an exploration phase in which the safe region is iteratively expanded, followed by an optimization phase in which Bayesian optimization is performed within the safe region. We assume that our algorithm runs for a fixed $T$ time steps, and that the first safe expansion region has horizon $T_0 < T$ with the optimization phase being $T_1 = T - T_0$ time steps long.

\algo models the utility function and the safety functions via Gaussian processes, and leverages their uncertainty in order to safely explore and optimize. In particular, at each iteration $t$, \algo uses the confidence intervals
\begin{align} \label{eq:qt}
  Q_t^i(\*x) \defeq \left[\mu_{t-1}^i(\*x) \pm \beta_t\sigma_{t-1}^i(\*x)\right],
\end{align}
where $\beta_t$ is a scalar whose  choice will be discussed later. We  use superscripts to denote the confidence intervals for the respective safety functions, and we use the superscript $f$ for the utility function. In order to guarantee both safety and progress in safe region expansion,  instead of using $Q_t^i$ directly, \algo uses the confidence intervals $C_t^i$ defined as $C_t^i(\*x) \defeq C_{t-1}^i(\*x) \cap Q_t^i(\*x)$, $C_0^i(\*x) = [h_i,\infty]$ so that $C_t^i$ are sequentially contained in $C_{t-1}^i$ for all $t$. We also define the upper and lower bounds of $C_t^i$ to be $u_t^i$ and $\ell_t^i$ respectively, as well as the width as $w_t^i = u_t^i-\ell_t^i$.

We defined the optimization goal with respect to a tolerance parameter $\epsilon$, which can be employed as a stopping condition for the expansion stage. Namely, if the expansion stage stops at $T_0$ under the condition 
$\forall i, \epsilon_t^i = \max_{\*x\in G_t} w_t(\*x) < \epsilon,$
then the $\epsilon$-Reachable safe region $\Reps^{T_0}(S_0)$ is guaranteed to be expanded. Similarly, we have a tolerance parameter $\zeta$ (in Algorithm~\ref{alg:safe}) to control utility function optimization with time horizon $T_1$.

\paragraph{Stage One: Safe region expansion.} \algo expands the safe region in the same way as that of \safeopt \citep{sui15icml,berkenkamp16bayesian}. An increasing sequence of safe subsets $S_t\subseteq D$ is computed based on the confidence intervals of the GP posterior:
$$S_t = \bigcap_{i}{\bigcup_{\*x \in S_{t-1}}\bigsetdef{\*x' \in D}{\ell_t^i(\*x) - L_i d(\*x, \*x') \geq h_i}}.$$
At each iteration, \algo computes a set of \emph{expander} points $G_t$ (that is, points within the current safe region that are likely to expand the safe region) and picks the expander with the highest predictive uncertainty.

In order to define the set $G_t$, we first define the function:
$$e_t(\*x) \defeq \Big|\bigcap_i \bigsetdef{\*x' \in D \setminus S_t}{u_t(\*x) - L_i d(\*x, \*x') \geq h_i}\Big|,$$
which (optimistically) quantifies the potential enlargement of the current safe set after we sample a new decision $\*x$. Then, $G_t$ is simply given by:
$$G_t=\{\*x\in S_t:e_t(\*x)>0\}.$$

Finally, at each iteration \algo selects $x_t$ to be $x_t = \argmax_{\*x\in G_t} w_n(\*x, i)$.

\paragraph{Stage Two: Utility optimization.} Once the safe region is established, \algo can use any standard online optimization approach to optimize the utility function within the expanded safe region.  For concreteness, we present here the GP-UCB algorithm \citep{srinivas10}.  For completeness, we present a version of \algo based on preference-based utility optimization in Appendix \ref{app:duel}.
Our theoretical analysis is also predicated on using GP-UCB, since it offers finite-time regret bounds. Formally, at each iteration in this phase, we select the arm $x_t$ as the following:
\begin{align} \label{eq:ucb}
x_t = \argmax_{\*x\in S_t} \mu_{t-1}^f(\*x) + \beta_t\sigma_{t-1}^f(\*x)
\end{align}
Note that it possible (though typically unlikely) for the safe region to further expand during this phase. 

\paragraph{Comparison between \safeopt and \algo.}
\begin{figure*}[t!]
\centering
\begin{subfigure}[t]{\dimexpr0.31\textwidth+20pt\relax}
    \makebox[20pt]{\raisebox{40pt}{\rotatebox[origin=c]{90}{\safeopt utility}}}%
    \includegraphics[width=\dimexpr\linewidth-20pt\relax]
    {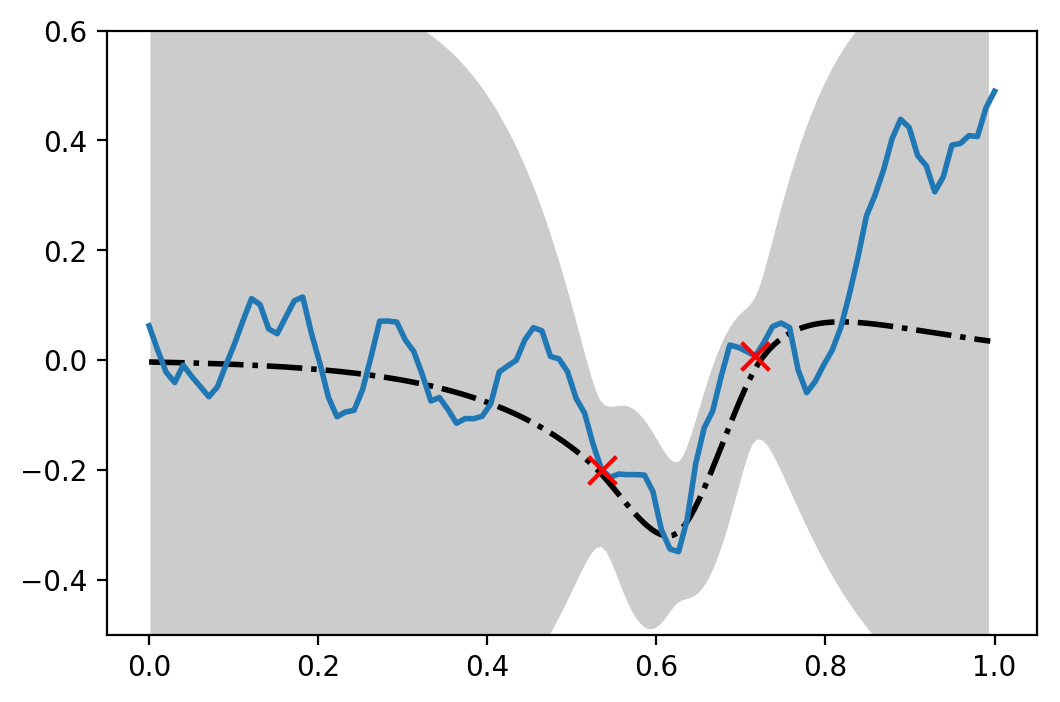}
    \makebox[20pt]{\raisebox{40pt}{\rotatebox[origin=c]{90}{\safeopt safety}}}%
    \includegraphics[width=\dimexpr\linewidth-20pt\relax]
    {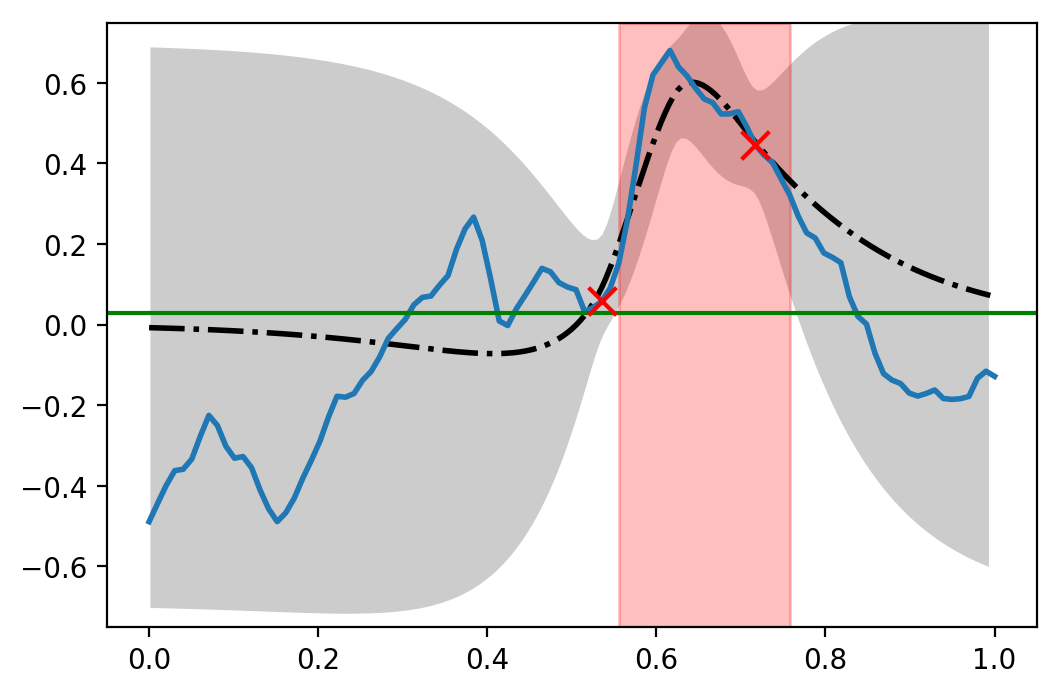}
    \makebox[20pt]{\raisebox{40pt}{\rotatebox[origin=c]{90}{\algo utility}}}%
    \includegraphics[width=\dimexpr\linewidth-20pt\relax]
    {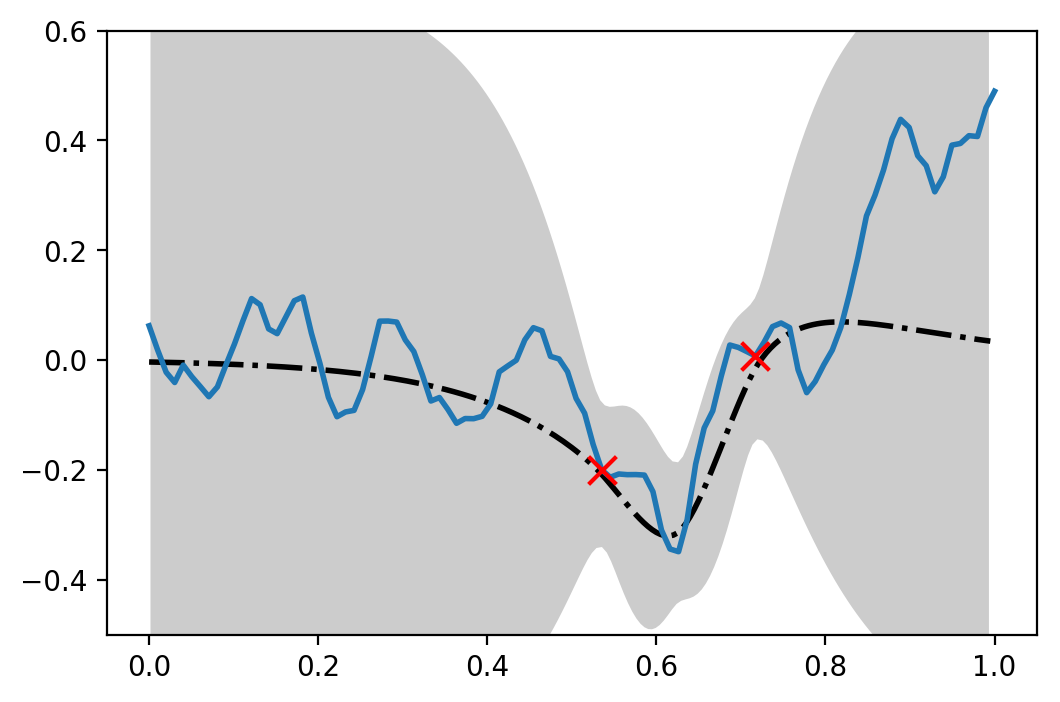}
    \makebox[20pt]{\raisebox{40pt}{\rotatebox[origin=c]{90}{\algo safety}}}%
    \includegraphics[width=\dimexpr\linewidth-20pt\relax]
    {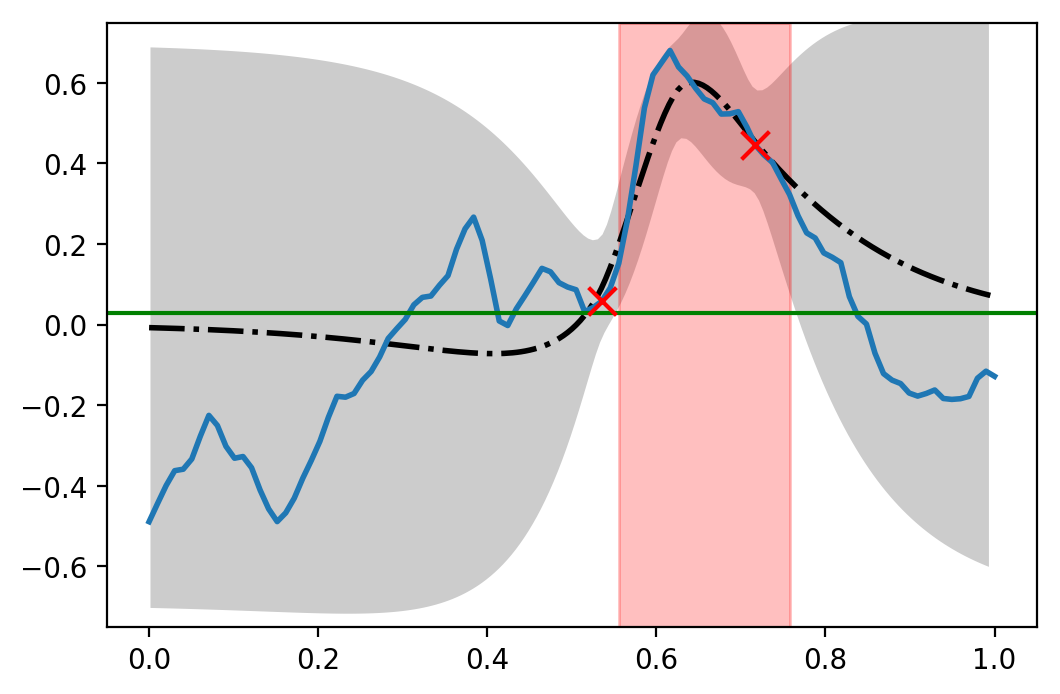}
    \caption{2 iterations}
\end{subfigure}\hfill
\begin{subfigure}[t]{0.31\textwidth}
    \includegraphics[width=\textwidth]  
    {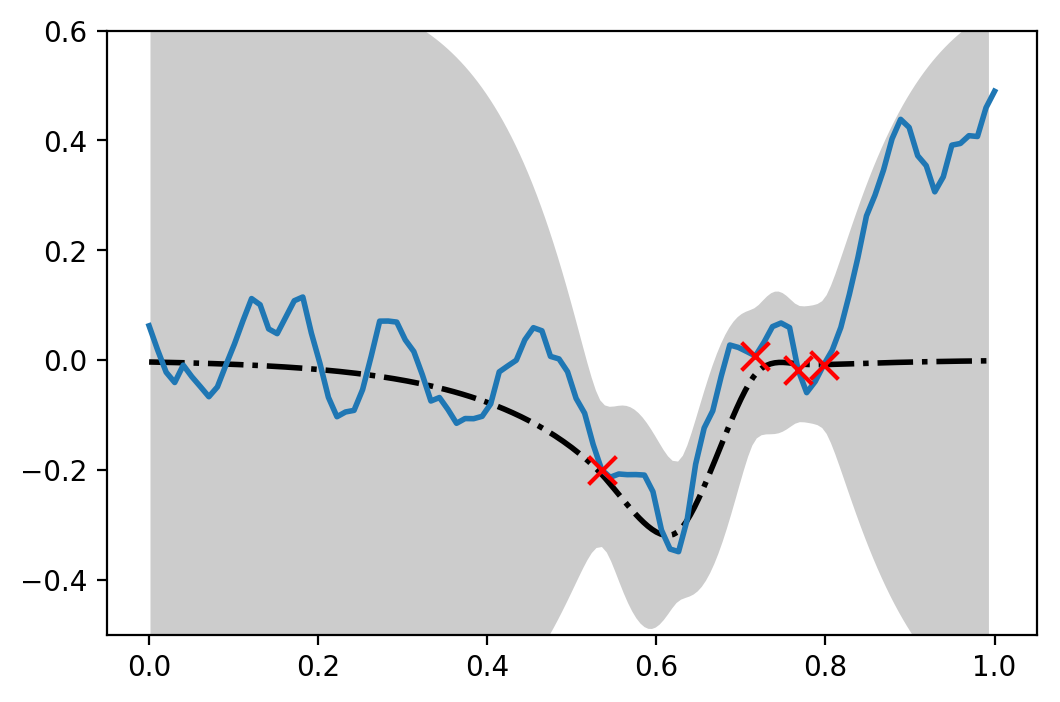}
    \includegraphics[width=\textwidth]
    {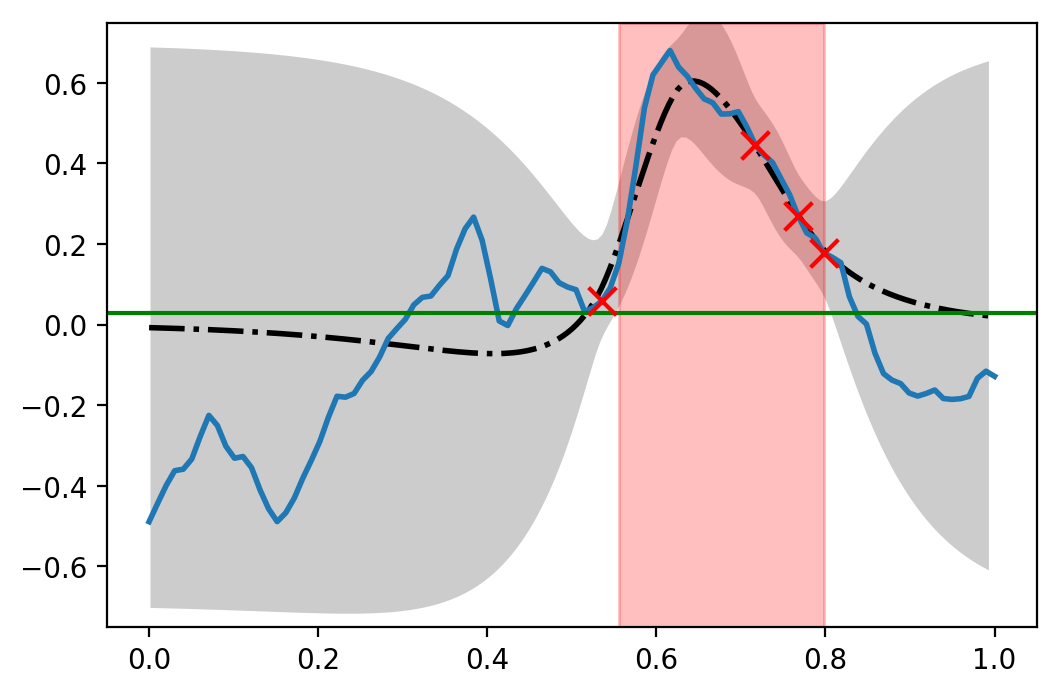}
    \includegraphics[width=\textwidth]
    {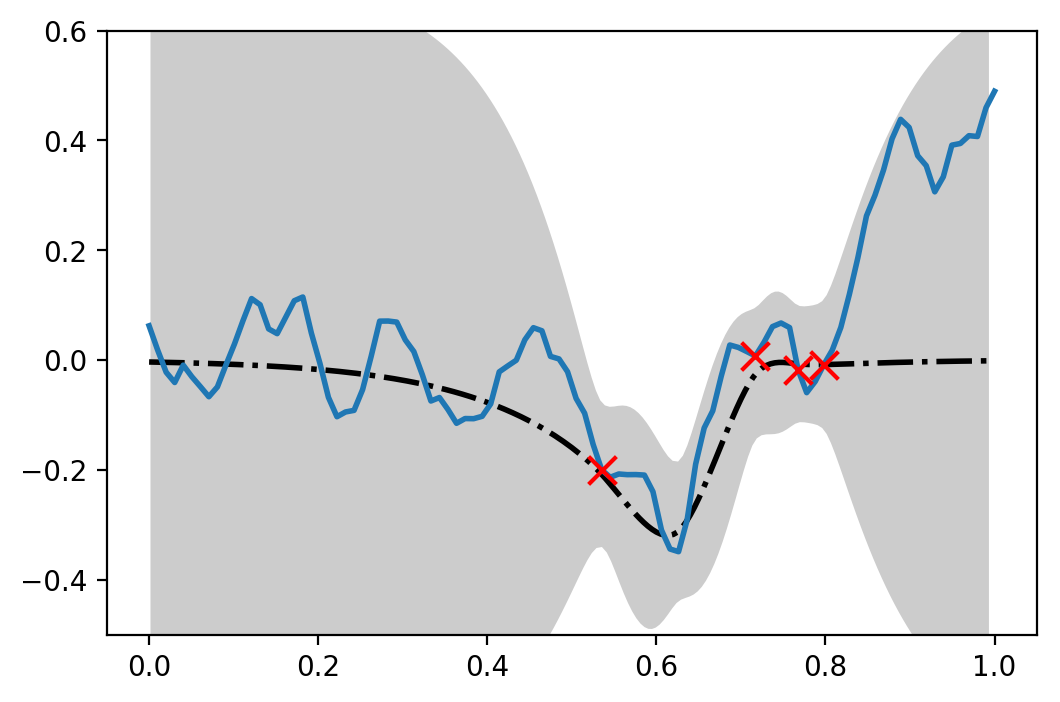}
    \includegraphics[width=\textwidth]
    {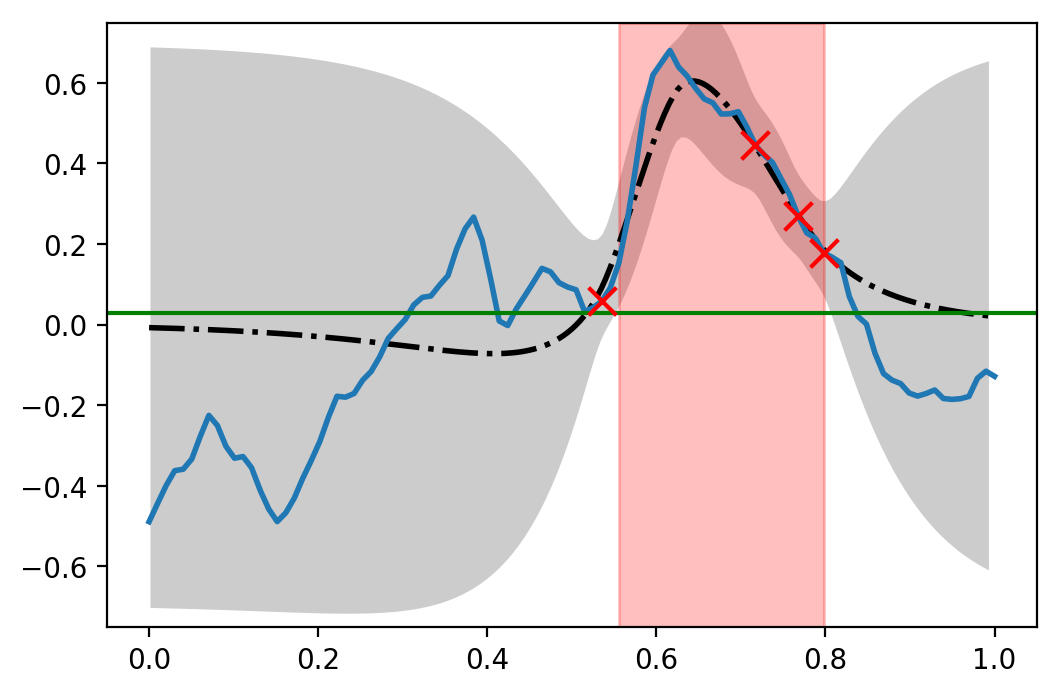}
    \caption{4 iterations}
\end{subfigure}\hfill
\begin{subfigure}[t]{0.31\textwidth}
    \includegraphics[width=\textwidth]  
    {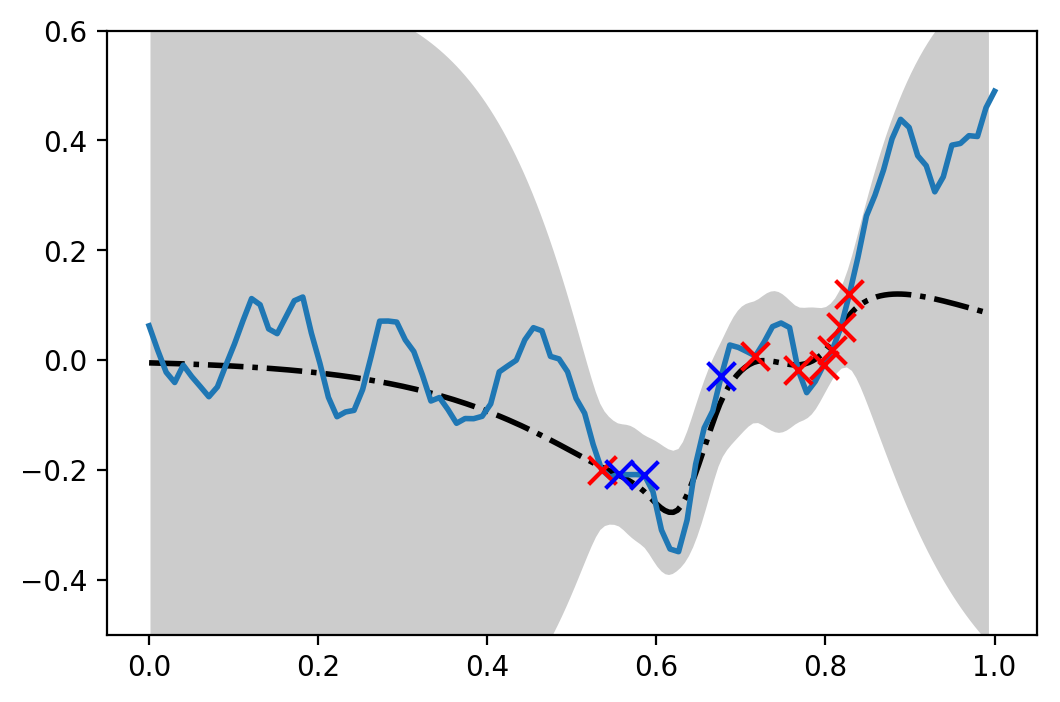}
    \includegraphics[width=\textwidth]
    {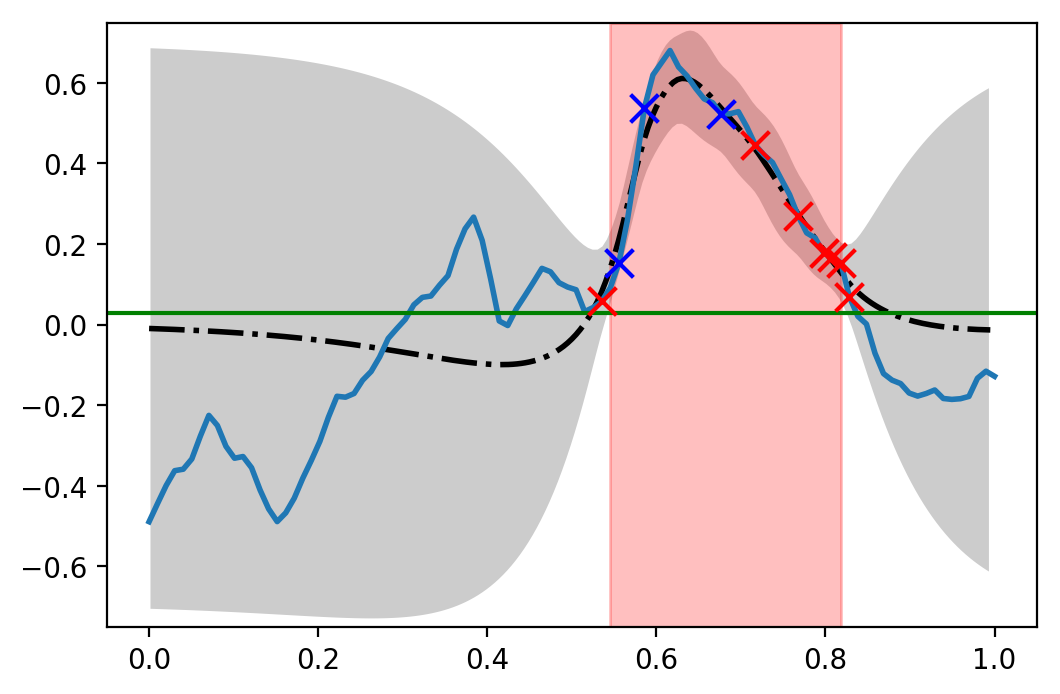}
    \includegraphics[width=\textwidth]
    {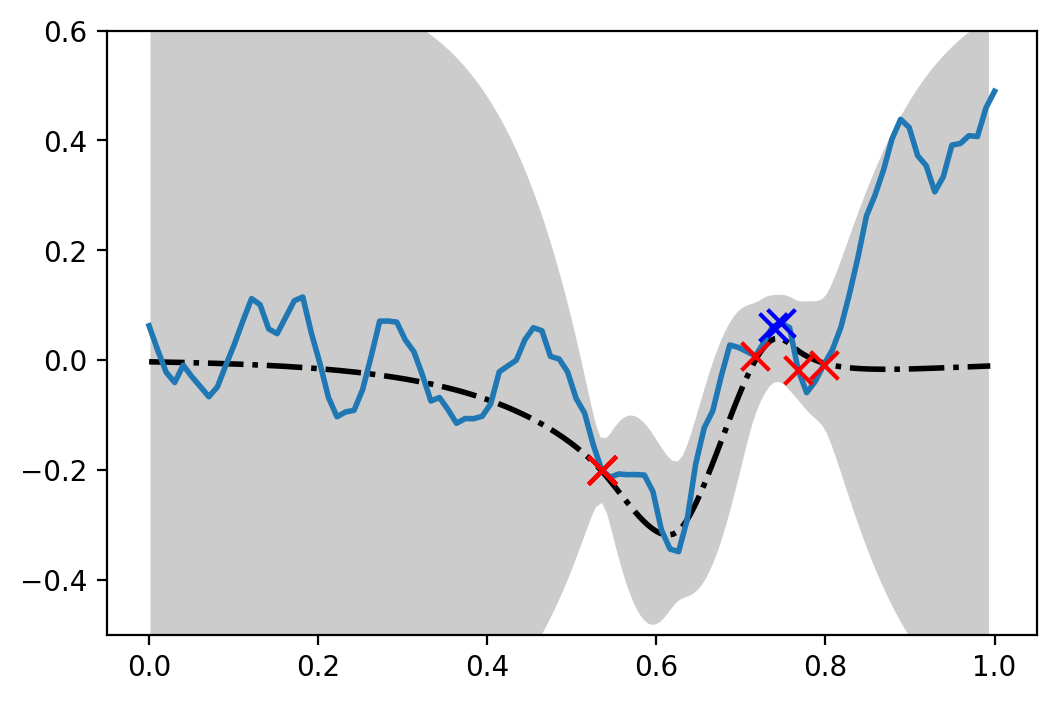}
    \includegraphics[width=\textwidth]
    {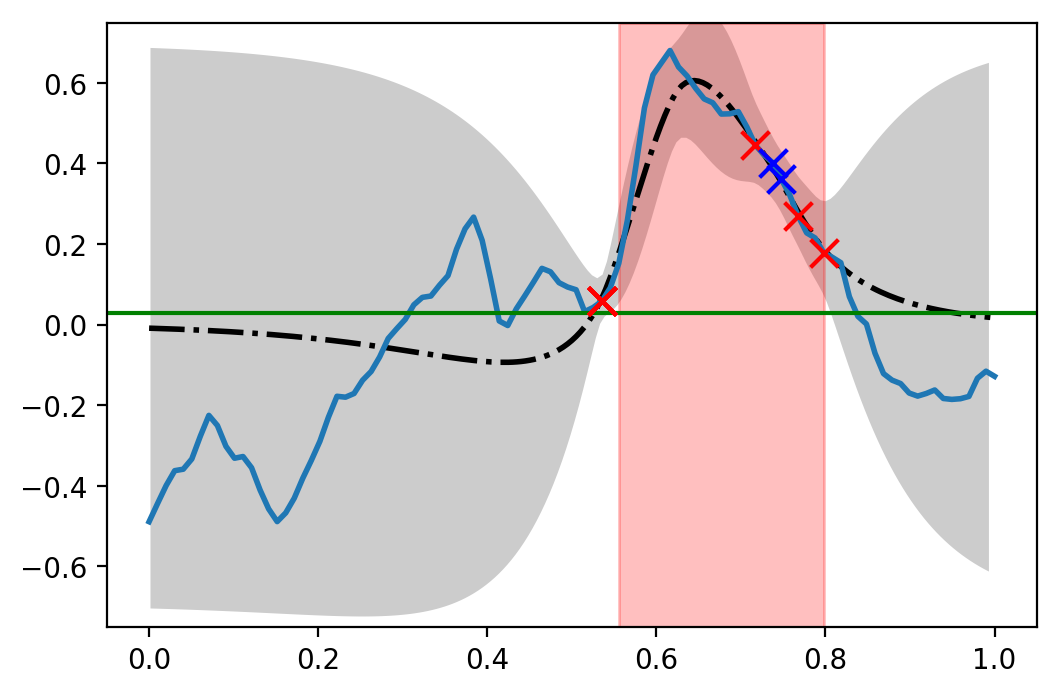}
    \caption{10 iterations}
\end{subfigure}
\caption{Evolution of GPs in \safeopt and \algo for a fixed safe seed; dashed lines correspond to the mean and shaded areas to $\pm 2$ standard deviations. The first and third rows depict the utility function, and the second and fourth rows depict a single safety function.  The  utility and safety functions were randomly sampled from a zero-mean GP with a Matern kernel, and are represented with solid blue lines. The safety threshold is shown as the green line, and safe regions are shown in red.  The red markers correspond to safe expansions and blue markers to maximizations and optimizations.  We see that \algo identifies actions with higher utility than \safeopt.}
\label{fig:compare}
\end{figure*}

Although \algo is similar to \safeopt in that it constructs confidence intervals and defines the safe region in the same way, there are distinct differences in how these algorithms work. We illustrate the behavior of \safeopt and \algo starting from a common safe seed in Figure \ref{fig:compare}. Initially, both algorithms select the same points since they use the same definition of safe expansion. However, \algo selects noticeably better optimization points than \safeopt due its UCB criterion. We leave a more detailed discussion of this behavior for Section \ref{sec:experiment}.

We also re-emphasize that since \algo separates the safe optimization problem into safe expansion and utility optimization phases, it is much more amenable to a variety of related settings than \safeopt. For example, as discussed in detail in the appendix, dueling feedback can easily be incorporated into \algo: in the dueling setting, one can simply replace GP-UCB in the  utility optimization stage with any kernelized dueling-bandit algorithm, such as \kersparring \citep{sui2017multi}.


\begin{algorithm}[tb]
  \caption{\algo}
  \label{alg:safe}
{\small
\begin{algorithmic}[1]
\STATE {\bfseries Input:} sample set $D$, $i\in\{1,\dots,n\}$,\\  
           \hspace{3.0em}GP prior for utility function $f$,\\  
           \hspace{3.0em}GP priors for safety functions $g_i$,\\
           \hspace{3.0em}Lipschitz constants $L_i$ for $g_i$,\\
           \hspace{3.0em}safe seed set $S_0$,\\
           \hspace{3.0em}safety threshold $h_i$,\\
           \hspace{3.0em}accuracies $\epsilon$ (for expansion), $\zeta$ (for optimization).
  \LET{$C_0^i(\*x)$}{$[h_i, \infty)$, for all $\*x \in S_0$}
  \LET{$C_0^i(\*x)$}{$\mathbb{R}$, for all $\*x \in D \setminus S_0$}
  \LET{$Q_0^i(\*x)$}{$\mathbb{R}$, for all $\*x \in D$}
  \LET{$C_0^f(\*x)$}{$\mathbb{R}$, for all $\*x \in D$}
  \LET{$Q_0^f(\*x)$}{$\mathbb{R}$, for all $\*x \in D$}

  \FOR{$t = 1, \ldots\, T_0$}
    \LET{$C_t^i(\*x)$}{$C_{t-1}^i(\*x) \cap Q_{t-1}^i(\*x)$}
    \LET{$C_t^f(\*x)$}{$C_{t-1}^f(\*x) \cap Q_{t-1}^f(\*x)$}
    \LET{$S_t$}{$\bigcap_{i}{\bigcup_{\*x \in S_{t-1}}\bigsetdef{\*x' \in D}{\ell_t^i(\*x) - L_i d(\*x, \*x') \geq h_i}}$}
    \LET{$G_t$}{$\bigsetdef{\*x \in S_t}{e_t(\*x) > 0}$}
    \IF{$\forall i, \epsilon_t^i < \epsilon$}
    \LET{$\*x_t$}{$\argmax_{\*x\in S_t} \mu_{t-1}^f(\*x) + \beta_t\sigma_{t-1}^f(\*x)$}
    \ELSE
    \LET{$\*x_t$}{$\argmax_{\*x \in G_t, i\in\{1,\dots,n\}}w_t^i(\*x)$} 
    \ENDIF
    \LET{$y_{f, t}$}{$f(\*x_t) + n_{f, t}$}
    \LET{$y_{i, t}$}{$g_i(\*x_t) + n_{i, t}$}
    \STATE Compute $Q_{f, t}(\*x)$ and $Q_{i, t}(\*x)$, for all $\*x \in S_t$
  \ENDFOR
  \FOR{$t=T_0+1, \ldots, T$}
    \LET{$C_t^f(\*x)$}{$C_{t-1}^f(\*x) \cap Q_{t-1}^f(\*x)$}
    \LET{$\*x_t$}{$\argmax_{\*x\in S_t} \mu_{t-1}^f(\*x) + \beta_t\sigma_{t-1}^f(\*x)$}
    \LET{$y_{f, t}$}{$f(\*x_t) + n_{f, t}$}
    \LET{$y_{i, t}$}{$g_i(\*x_t) + n_{i, t}$}
    \STATE Compute $Q_{f, t}(\*x)$ and $Q_{i, t}(\*x)$, for all $\*x \in S_t$ 
  \ENDFOR
\end{algorithmic}
}
\end{algorithm}

\section{Theoretical Results}
\label{sec:theoretical}

In this section, we show the effectiveness of \algo by theoretically bounding its sample complexity for expansion and optimization. The two stages of \algo are the expansion of the safe region in search for the total safe region, and the optimization within the safe region.

The correctness of \algo relies on the fact that the classification of sets $S_t$ and $G_t$ is sound. While this requires that the confidence bounds $C_t$ are conservative, using bounds that are too conservative will slow down the algorithm considerably. The tightness of the confidence bounds is controlled by parameter $\beta_t$ in Equation~\ref{eq:qt}. This problem of properly tuning confidence bounds using Gaussian processes in exploration--exploitation trade-off has been studied by \citet{srinivas10,chowdhury2017kernelized}. These algorithms are designed for the stochastic multi-armed bandit problem on a kernelized input space without safety constraints. However, their choice of confidence bounds can be generalized to our setting for expansion and optimization. In particular, for our theoretical results to hold it suffices to choose:
\begin{align} \label{eq:beta}
  \beta_t = B + \sigma \sqrt{2(\gamma_{t-1} + 1 +\log(1/\delta))},
\end{align}
where $B$ is a bound on the RKHS norm of $f$, $\delta$ is the allowed failure probability, observation noise is $\sigma$-sub-Gaussian, and $\gamma_t$ quantifies the effective degrees of freedom associated with the kernel function. Concretely, 
$$\gamma_t=\max_{|A|\leq t} I(f;\*y_A)$$ is the maximal mutual information that can be obtained about the GP prior from $t$ samples.


We present two main theorems for \algo. Theorem~\ref{thm:safe} ensures  convergence to the reachable safe region in the safe expansion stage. Theorem~\ref{thm:opt} ensures  convergence towards optimal utility value within the safe region in the utility optimization stage. Both results are finite time bounds.

\begin{theorem}
\label{thm:safe}
Suppose safety functions $g_i$ satisfies $\|g_i\|^2_k \leq B$ and $g_i$ further is $L_i$-Lipschitz-continuous. $i\in \{1, \dots, n\}$. Also, suppose $S_0\neq \varnothing$, and $g_i(\*x)\geq h_i$, for all $\*x\in S_0$.  Fix any $\epsilon>0$ and $\delta \in (0, 1)$.
Suppose we run the safe region expansion stage of \algo with seed set $S_0$, with noise $n_t$ to be $\sigma$-sub-Gaussian, and $\beta_t = B + \sigma \sqrt{2(\gamma_{t-1} + 1 +\log(1/\delta))}$ with safety function hyperparameters. Let $t^*$ be the smallest positive integer satisfying
\begin{align*}
\frac{t^*}{\beta_{t^*}^2 \gamma_{nt^*}} \geq \frac{C_1 \left(|\Rbo(S_0)| + 1\right)}{\epsilon^2},
\end{align*}
where $C_1 = 8 / \log(1 + \sigma^{-2})$.
Then, the following jointly hold with probability at least $1-\delta$:
\begin{itemize}
  \item $\forall t \geq 1$ and $i \in \{1, \dots, n\}$, $g_i(\*x_t) \geq h_i$,
  \item $\forall t \geq t^*$, $\epsilon$-Reachable safe region $\Reps^{T_0}(S_0)$ is guaranteed to be expanded.
\end{itemize}
\end{theorem}

The detailed proof of Theorem~\ref{thm:safe} is presented in Appendix~\ref{app:b}. In Theorem~\ref{thm:safe}, we count $t$ from the beginning of expansion stage. We choose $T_0 = t^*$ with $T_0$ the expansion time defined in Section~\ref{sec:algorithm}. We show that with high probability, the expansion stage of \algo guarantees safety, and expands the initial safe region $S_0$ to an $\epsilon$-reachable set after at most $t^*$ iterations. The size of $t^*$ depends on the largest size of safe region $\Rbo(S_0)$, the accuracy parameters $\epsilon$, $\zeta$, the confidence parameter $\delta$, the complexity of the function $B$ and the parameterization of the GP via $\gamma_t$.

The proof is based on the following idea. Within a stage, wherein $S_t$ does not expand, the uncertainty $w_t(\*x_t)$ monotonically decreases due to construction of $G_t$. We prove that, the condition $\max_{\*x\in G_t}w(\*x)<\epsilon$ implies either of two possibilities: $S_t$ will expand after the next evaluation, i.e., the reachable region will increase, and, hence, the next stage shall commence; or, we have already established all decisions within $\Rbeps(S_0)$ as safe, i.e., $S_t = \Rbeps(S_0)$. To establish the sample complexity we use a bound on how quickly $w_t(\*x_t)$ decreases.

\begin{theorem}
\label{thm:opt}
Suppose utility function $f$ satisfies $\|f\|^2_k \leq B$, $\delta \in (0,1)$, and noise $n_t$ is $\sigma$-sub-Gaussian. $\beta_t = B + \sigma \sqrt{2(\gamma_{t-1} + 1 +\log(1/\delta))}$ with utility function hyperparameters.  $T_1$ the time horizon for optimization stage. Fix any $\zeta>0$.
Suppose we run the optimization stage of \algo within the expansion stage safe region $\Reps^{T_0}(S_0)$.
Let $Y$ be the smallest positive integer satisfying 
$$\frac{4\sqrt{2}}{\sqrt{Y}}(B\sqrt{\gamma_Y} + \sigma \sqrt{2\gamma_Y(\gamma_Y + 1 +\log(1/\delta))}) \leq \zeta$$
Then with probability at least $1-\delta$, \algo finds $\zeta$-optimal utility value: $f(\hat{\*x}^*) \geq f(\*x^*) - \zeta$.
\end{theorem}

The proof of Theorem~\ref{thm:opt} is  presented in Appendix~\ref{app:b}. We count $t$ from the beginning of the optimization stage in Theorem~\ref{thm:opt}. We  choose $T_1 = Y$ with $T_1$ the time horizon of optimization stage. We prove the existence of an $\epsilon$-optimal decision $\hat{\*x}^*$ within the expansion stage safe region. 

\textbf{Discussion.} \algo separates safe region expansion and utility function maximization into two distinct stages. Theorem~\ref{thm:safe} guarantees $\epsilon$-optimal expansion in the first stage within time horizon $T_0$. Theorem~\ref{thm:opt} guarantees $\zeta$-optimal utility value in the second stage within time horizon $T_1$. Compared to existing approaches which interleave between expansion and optimization, \algo does not require any similarity or comparability between safety and utility. In Section~\ref{sec:experiment} we show empirically that \algo is more efficient and far more natural for some applications.
\section{Experimental Results}
\label{sec:experiment}

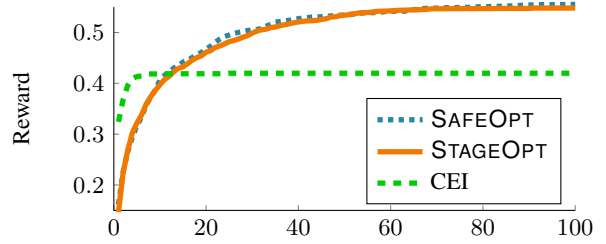
\begin{figure}[t]
\centering
  \setlength\figureheight{0.5\textwidth}
  \setlength\figurewidth{0.45\textwidth}
\begin{tikzpicture}
\smallskip

\begin{axis}[
tick label style={font=\footnotesize},
label style={font=\footnotesize},
legend style={font=\footnotesize},
legend pos=south east,
view={0}{90},
width=\figurewidth,
height=\figureheight/2,
ylabel={Reward},
xmin=0, xmax=100,
xtick={0, 20, 40, 60, 80, 100},
ymin=0.15, ymax=0.55,
major tick length=2pt,
axis lines*=left,
legend cell align=left,
clip=false]


\addplot [
color=cyan!60!black,
dotted,
line width=2pt,
]
coordinates{
(1, 0.161576)(2, 0.224962)(3, 0.264920)(4, 0.293574)(5, 0.316228)(6, 0.333617)(7, 0.357171)(8, 0.373100)(9, 0.386949)(10, 0.406621)(11, 0.415690)(12, 0.422297)(13, 0.430128)(14, 0.434806)(15, 0.441893)(16, 0.445416)(17, 0.450754)(18, 0.455955)(19, 0.460989)(20, 0.467090)(21, 0.473311)(22, 0.481170)(23, 0.487176)(24, 0.492277)(25, 0.495008)(26, 0.498042)(27, 0.500286)(28, 0.501854)(29, 0.503912)(30, 0.505059)(31, 0.505956)(32, 0.508032)(33, 0.512380)(34, 0.516055)(35, 0.520434)(36, 0.521631)(37, 0.523089)(38, 0.523866)(39, 0.525419)(40, 0.527189)(41, 0.528644)(42, 0.529376)(43, 0.530522)(44, 0.530818)(45, 0.530936)(46, 0.531735)(47, 0.532532)(48, 0.533349)(49, 0.533958)(50, 0.534322)(51, 0.535560)(52, 0.536095)(53, 0.537821)(54, 0.538681)(55, 0.538809)(56, 0.539067)(57, 0.539149)(58, 0.540188)(59, 0.541435)(60, 0.541807)(61, 0.541845)(62, 0.541845)(63, 0.541998)(64, 0.543347)(65, 0.545020)(66, 0.545234)(67, 0.547157)(68, 0.547563)(69, 0.547682)(70, 0.548066)(71, 0.548066)(72, 0.548865)(73, 0.548956)(74, 0.549000)(75, 0.549139)(76, 0.549352)(77, 0.549526)(78, 0.550549)(79, 0.550549)(80, 0.550872)(81, 0.550875)(82, 0.551319)(83, 0.551400)(84, 0.551596)(85, 0.552278)(86, 0.552278)(87, 0.552737)(88, 0.553220)(89, 0.554043)(90, 0.554102)(91, 0.554563)(92, 0.554820)(93, 0.554820)(94, 0.554820)(95, 0.554820)(96, 0.554906)(97, 0.555015)(98, 0.555042)(99, 0.555223)(100, 0.555225)
};

\addplot [
color=orange!95!black,
solid,
line width=2pt,
]
coordinates{
(1, 0.146507)(2, 0.223171)(3, 0.267792)(4, 0.302188)(5, 0.321759)(6, 0.337551)(7, 0.358793)(8, 0.373589)(9, 0.386291)(10, 0.397339)(11, 0.407536)(12, 0.414035)(13, 0.423018)(14, 0.429413)(15, 0.432441)(16, 0.438914)(17, 0.445634)(18, 0.450712)(19, 0.454298)(20, 0.461342)(21, 0.464492)(22, 0.470713)(23, 0.473514)(24, 0.477557)(25, 0.482088)(26, 0.484232)(27, 0.487274)(28, 0.490784)(29, 0.493968)(30, 0.498776)(31, 0.502888)(32, 0.505305)(33, 0.506585)(34, 0.509670)(35, 0.510901)(36, 0.513476)(37, 0.515250)(38, 0.517193)(39, 0.519216)(40, 0.520725)(41, 0.522138)(42, 0.522855)(43, 0.523139)(44, 0.523958)(45, 0.525782)(46, 0.527115)(47, 0.530269)(48, 0.530918)(49, 0.533154)(50, 0.534184)(51, 0.535053)(52, 0.536588)(53, 0.538368)(54, 0.538543)(55, 0.539293)(56, 0.540260)(57, 0.541684)(58, 0.541875)(59, 0.542122)(60, 0.542690)(61, 0.543207)(62, 0.543806)(63, 0.543917)(64, 0.544256)(65, 0.544784)(66, 0.544811)(67, 0.545188)(68, 0.545739)(69, 0.546620)(70, 0.546745)(71, 0.546745)(72, 0.546872)(73, 0.547020)(74, 0.547020)(75, 0.547020)(76, 0.547020)(77, 0.547020)(78, 0.547020)(79, 0.547020)(80, 0.547020)(81, 0.547020)(82, 0.547020)(83, 0.547397)(84, 0.547397)(85, 0.547397)(86, 0.547397)(87, 0.547397)(88, 0.547397)(89, 0.547397)(90, 0.547397)(91, 0.548134)(92, 0.548134)(93, 0.548134)(94, 0.548134)(95, 0.548134)(96, 0.548134)(97, 0.548134)(98, 0.548134)(99, 0.548134)(100, 0.548134)
};

\addplot [
color=green!80!black,
dashed,
line width=2pt,
]
coordinates{
(1, 0.324199)(2, 0.367732)(3, 0.395531)(4, 0.408990)(5, 0.413465)(6, 0.414437)(7, 0.418388)(8, 0.418494)(9, 0.418494)(10, 0.418494)(11, 0.418777)(12, 0.418885)(13, 0.418885)(14, 0.418885)(15, 0.418885)(16, 0.418885)(17, 0.418885)(18, 0.418885)(19, 0.418885)(20, 0.418885)(21, 0.418885)(22, 0.418885)(23, 0.418885)(24, 0.418885)(25, 0.419891)(26, 0.419891)(27, 0.419891)(28, 0.419891)(29, 0.419891)(30, 0.419891)(31, 0.419891)(32, 0.419891)(33, 0.419891)(34, 0.419891)(35, 0.419891)(36, 0.419891)(37, 0.419891)(38, 0.419891)(39, 0.419891)(40, 0.419891)(41, 0.419891)(42, 0.419891)(43, 0.419891)(44, 0.419891)(45, 0.419891)(46, 0.419891)(47, 0.419891)(48, 0.419891)(49, 0.419891)(50, 0.419891)(51, 0.419891)(52, 0.419891)(53, 0.419891)(54, 0.419891)(55, 0.419891)(56, 0.419891)(57, 0.419891)(58, 0.419891)(59, 0.419891)(60, 0.419891)(61, 0.419891)(62, 0.419891)(63, 0.419891)(64, 0.419891)(65, 0.419891)(66, 0.419891)(67, 0.419891)(68, 0.419891)(69, 0.419891)(70, 0.419891)(71, 0.419891)(72, 0.419891)(73, 0.419891)(74, 0.419891)(75, 0.419891)(76, 0.419891)(77, 0.419891)(78, 0.419891)(79, 0.419891)(80, 0.419891)(81, 0.419891)(82, 0.419891)(83, 0.419891)(84, 0.419891)(85, 0.419891)(86, 0.419891)(87, 0.419891)(88, 0.419891)(89, 0.419891)(90, 0.419891)(91, 0.419891)(92, 0.419891)(93, 0.419891)(94, 0.419891)(95, 0.419891)(96, 0.419891)(97, 0.419891)(98, 0.419891)(99, 0.419891)(100, 0.419891)
};
\legend{\safeopt, \algo, CEI}




\end{axis}

\end{tikzpicture}
\caption{Comparison between \safeopt, \algo, and constrained EI on synthetic data with one safety function.  In this simple setting, both \safeopt and \algo perform similarly. In order to achieve the same level of safety guarantees, constrained EI must be much more careful during exploration, and consequently fails to identify the optimal point.}
\label{fig:cei_compare}
\end{figure}

\begin{figure*}[t!]
  \centering
  \setlength\figureheight{0.45\textwidth}
  \setlength\figurewidth{0.3\textwidth}
  \begin{subfigure}[t]{0.3\textwidth}
    \begin{tikzpicture}

\begin{axis}[
tick label style={font=\footnotesize},
label style={font=\footnotesize},
legend style={nodes={scale=0.6, transform shape}},
legend pos=south east,
view={0}{90},
width=\figurewidth,
height=\figureheight/2,
ylabel={Reward},
xmin=40, xmax=100,
xtick={40, 60, 80, 100},
ymin=2.04, ymax=2.15,
major tick length=2pt,
axis lines*=left,
legend cell align=left,
clip=false]


\addplot [
color=cyan!60!black,
dotted,
line width=2pt,
]
coordinates{
(40, 2.058297)(41, 2.070389)(42, 2.072232)(43, 2.072232)(44, 2.092996)(45, 2.096473)(46, 2.099558)(47, 2.104734)(48, 2.107040)(49, 2.110187)(50, 2.110187)(51, 2.110187)(52, 2.113664)(53, 2.114343)(54, 2.116883)(55, 2.116883)(56, 2.118919)(57, 2.121041)(58, 2.121093)(59, 2.121093)(60, 2.121093)(61, 2.123862)(62, 2.123862)(63, 2.123862)(64, 2.123862)(65, 2.123862)(66, 2.123862)(67, 2.123862)(68, 2.123862)(69, 2.123862)(70, 2.123862)(71, 2.123862)(72, 2.123862)(73, 2.123862)(74, 2.123862)(75, 2.123862)(76, 2.123862)(77, 2.126494)(78, 2.126934)(79, 2.126934)(80, 2.127061)(81, 2.127061)(82, 2.127061)(83, 2.127061)(84, 2.127061)(85, 2.127061)(86, 2.127061)(87, 2.127501)(88, 2.127501)(89, 2.127501)(90, 2.127501)(91, 2.127501)(92, 2.127501)(93, 2.127501)(94, 2.127501)(95, 2.132910)(96, 2.135120)(97, 2.135120)(98, 2.135120)(99, 2.142284)(100, 2.149869)
};

\addplot [
color=orange!95!black,
solid,
line width=2pt,
]
coordinates{
(40, 2.127843)(41, 2.127843)(42, 2.127843)(43, 2.127843)(44, 2.127843)(45, 2.129879)(46, 2.129879)(47, 2.129879)(48, 2.129879)(49, 2.129879)(50, 2.129879)(51, 2.129879)(52, 2.129879)(53, 2.129879)(54, 2.129879)(55, 2.129879)(56, 2.129879)(57, 2.129879)(58, 2.140199)(59, 2.140199)(60, 2.140199)(61, 2.140199)(62, 2.140199)(63, 2.140199)(64, 2.140199)(65, 2.140199)(66, 2.140199)(67, 2.140199)(68, 2.140199)(69, 2.140199)(70, 2.140199)(71, 2.140199)(72, 2.140199)(73, 2.142491)(74, 2.142491)(75, 2.142491)(76, 2.142491)(77, 2.142491)(78, 2.142491)(79, 2.142491)(80, 2.142491)(81, 2.142491)(82, 2.142491)(83, 2.142491)(84, 2.142491)(85, 2.142491)(86, 2.142491)(87, 2.142491)(88, 2.142491)(89, 2.142491)(90, 2.142491)(91, 2.142491)(92, 2.142491)(93, 2.142491)(94, 2.142491)(95, 2.142491)(96, 2.142491)(97, 2.142491)(98, 2.142491)(99, 2.142491)(100, 2.142491)
};
\legend{\safeopt, \algo}




\end{axis}

\end{tikzpicture}
    \caption{1 safety function reward}
  \end{subfigure}
  \hfill
  \begin{subfigure}[t]{0.3\textwidth}
    \begin{tikzpicture}

\begin{axis}[
tick label style={font=\footnotesize},
label style={font=\footnotesize},
legend style={nodes={scale=0.6, transform shape}},
legend pos=south east,
view={0}{90},
width=\figurewidth,
height=\figureheight/2,
xmin=40, xmax=100,
xtick={40, 60, 80, 100},
ymin=1.0, ymax=1.3,
ylabel={Reward},
major tick length=2pt,
axis lines*=left,
legend cell align=left,
clip=false]


\addplot [
color=cyan!60!black,
dotted,
line width=2pt,
]
coordinates{
(40, 1.050902)(41, 1.052510)(42, 1.052753)(43, 1.058385)(44, 1.061063)(45, 1.067788)(46, 1.076085)(47, 1.077965)(48, 1.080115)(49, 1.080115)(50, 1.082709)(51, 1.092014)(52, 1.092014)(53, 1.094016)(54, 1.106885)(55, 1.106885)(56, 1.106885)(57, 1.106885)(58, 1.109001)(59, 1.111954)(60, 1.114727)(61, 1.114727)(62, 1.115157)(63, 1.115157)(64, 1.115332)(65, 1.116097)(66, 1.120665)(67, 1.123391)(68, 1.123940)(69, 1.124701)(70, 1.125462)(71, 1.125462)(72, 1.130001)(73, 1.130001)(74, 1.130021)(75, 1.130021)(76, 1.131081)(77, 1.131081)(78, 1.133963)(79, 1.136799)(80, 1.139497)(81, 1.140440)(82, 1.140440)(83, 1.140440)(84, 1.140440)(85, 1.140440)(86, 1.141185)(87, 1.141185)(88, 1.141185)(89, 1.143964)(90, 1.143964)(91, 1.144513)(92, 1.144513)(93, 1.144513)(94, 1.144513)(95, 1.144513)(96, 1.144513)(97, 1.144513)(98, 1.144513)(99, 1.144513)(100, 1.144513)
};

\addplot [
color=orange!95!black,
solid,
line width=2pt,
]
coordinates{
(40, 1.240061)(41, 1.241003)(42, 1.241003)(43, 1.241979)(44, 1.242687)(45, 1.244056)(46, 1.244056)(47, 1.244056)(48, 1.244605)(49, 1.247440)(50, 1.254360)(51, 1.254360)(52, 1.254360)(53, 1.254360)(54, 1.254360)(55, 1.254360)(56, 1.254360)(57, 1.254360)(58, 1.254360)(59, 1.254360)(60, 1.254360)(61, 1.254360)(62, 1.254360)(63, 1.254360)(64, 1.254360)(65, 1.254360)(66, 1.254360)(67, 1.254360)(68, 1.254360)(69, 1.254360)(70, 1.254360)(71, 1.254360)(72, 1.254360)(73, 1.254360)(74, 1.254360)(75, 1.254360)(76, 1.254360)(77, 1.254360)(78, 1.254360)(79, 1.254360)(80, 1.254360)(81, 1.254360)(82, 1.254360)(83, 1.254360)(84, 1.254360)(85, 1.254360)(86, 1.254360)(87, 1.254360)(88, 1.254360)(89, 1.254360)(90, 1.254360)(91, 1.254360)(92, 1.254360)(93, 1.254360)(94, 1.254360)(95, 1.254360)(96, 1.254360)(97, 1.254360)(98, 1.254360)(99, 1.254360)(100, 1.254360)
};
\legend{\safeopt, \algo}




\end{axis}

\end{tikzpicture}
    \caption{3 safety functions reward}
  \end{subfigure}
  \hfill
  \begin{subfigure}[t]{0.3\textwidth}
    \begin{tikzpicture}

\begin{axis}[
tick label style={font=\footnotesize},
label style={font=\footnotesize},
legend style={nodes={scale=0.6, transform shape}},
legend pos=south east,
view={0}{90},
width=\figurewidth,
height=\figureheight/2,
xmin=40, xmax=100,
xtick={40, 60, 80, 100},
ymin=2.6, ymax=2.69,
ylabel={Reward},
major tick length=2pt,
axis lines*=left,
legend cell align=left,
clip=false]


\addplot [
color=cyan!60!black,
dotted,
line width=2pt,
]
coordinates{
(40, 2.613479)(41, 2.614409)(42, 2.620247)(43, 2.620786)(44, 2.626781)(45, 2.628340)(46, 2.628340)(47, 2.630370)(48, 2.631850)(49, 2.632808)(50, 2.634249)(51, 2.635127)(52, 2.637928)(53, 2.642237)(54, 2.642850)(55, 2.643525)(56, 2.646342)(57, 2.646754)(58, 2.650951)(59, 2.650951)(60, 2.652199)(61, 2.658405)(62, 2.659650)(63, 2.661631)(64, 2.663857)(65, 2.664781)(66, 2.665499)(67, 2.665499)(68, 2.666259)(69, 2.669478)(70, 2.669764)(71, 2.669764)(72, 2.672771)(73, 2.672771)(74, 2.672771)(75, 2.672771)(76, 2.674034)(77, 2.674560)(78, 2.674560)(79, 2.674560)(80, 2.674560)(81, 2.677361)(82, 2.679586)(83, 2.679870)(84, 2.684229)(85, 2.684229)(86, 2.684365)(87, 2.684365)(88, 2.684467)(89, 2.685227)(90, 2.685227)(91, 2.685227)(92, 2.685227)(93, 2.685227)(94, 2.685227)(95, 2.685459)(96, 2.685459)(97, 2.685643)(98, 2.685643)(99, 2.685643)(100, 2.685643)
};
\legend{\safeopt, \algo}

\addplot [
color=orange!95!black,
solid,
line width=2pt,
]
coordinates{
(40, 2.617215)(41, 2.620823)(42, 2.620823)(43, 2.625218)(44, 2.625218)(45, 2.629327)(46, 2.629327)(47, 2.629611)(48, 2.629920)(49, 2.630358)(50, 2.636554)(51, 2.639011)(52, 2.645969)(53, 2.647318)(54, 2.649916)(55, 2.651413)(56, 2.657366)(57, 2.657366)(58, 2.657366)(59, 2.657366)(60, 2.658603)(61, 2.660831)(62, 2.660831)(63, 2.660831)(64, 2.661908)(65, 2.661908)(66, 2.661908)(67, 2.661908)(68, 2.661908)(69, 2.661908)(70, 2.664133)(71, 2.664807)(72, 2.664807)(73, 2.664807)(74, 2.665093)(75, 2.665093)(76, 2.667227)(77, 2.667987)(78, 2.667987)(79, 2.667987)(80, 2.667987)(81, 2.667987)(82, 2.670212)(83, 2.670212)(84, 2.670212)(85, 2.670212)(86, 2.670212)(87, 2.670212)(88, 2.670212)(89, 2.670212)(90, 2.670212)(91, 2.670212)(92, 2.670212)(93, 2.670212)(94, 2.670212)(95, 2.670212)(96, 2.670212)(97, 2.670212)(98, 2.670212)(99, 2.670212)(100, 2.670212)
};




\end{axis}

\end{tikzpicture}
    \caption{1 safety function, dueling feedback reward}
  \end{subfigure}
  \\
  \begin{subfigure}[t]{0.3\textwidth}
    \begin{tikzpicture}

\begin{axis}[
tick label style={font=\footnotesize},
label style={font=\footnotesize},
legend style={nodes={scale=0.6, transform shape}},
legend pos=south east,
view={0}{90},
width=\figurewidth,
height=\figureheight/2,
ylabel=Safe region size,
xmin=40, xmax=100,
xtick={40, 60, 80, 100},
ymin=50, ymax=80,
major tick length=2pt,
axis lines*=left,
legend cell align=left,
clip=false]


\addplot [
color=cyan!60!black,
dotted,
line width=2pt,
]
coordinates{
(40, 59.959459)(41, 60.189189)(42, 60.364865)(43, 60.500000)(44, 60.756757)(45, 60.864865)(46, 60.986486)(47, 61.162162)(48, 61.378378)(49, 61.594595)(50, 61.689189)(51, 61.918919)(52, 62.162162)(53, 62.270270)(54, 62.540541)(55, 62.648649)(56, 62.959459)(57, 63.162162)(58, 63.445946)(59, 63.621622)(60, 63.837838)(61, 64.229730)(62, 64.391892)(63, 64.567568)(64, 64.810811)(65, 64.932432)(66, 65.067568)(67, 65.175676)(68, 65.189189)(69, 65.324324)(70, 65.540541)(71, 65.635135)(72, 65.756757)(73, 65.864865)(74, 65.905405)(75, 66.054054)(76, 66.216216)(77, 66.270270)(78, 66.472973)(79, 66.729730)(80, 66.729730)(81, 66.878378)(82, 66.891892)(83, 66.986486)(84, 67.040541)(85, 67.094595)(86, 67.162162)(87, 67.162162)(88, 67.297297)(89, 67.472973)(90, 67.567568)(91, 67.608108)(92, 67.581081)(93, 67.648649)(94, 67.716216)(95, 67.810811)(96, 67.891892)(97, 67.891892)(98, 68.000000)(99, 68.040541)(100, 68.094595)
};

\addplot [
color=orange!95!black,
solid,
line width=2pt,
]
coordinates{
(40, 72.594595)(41, 72.797297)(42, 73.013514)(43, 73.405405)(44, 73.337838)(45, 73.310811)(46, 73.567568)(47, 73.689189)(48, 73.797297)(49, 73.864865)(50, 73.878378)(51, 73.837838)(52, 73.797297)(53, 73.864865)(54, 74.013514)(55, 74.094595)(56, 74.162162)(57, 74.189189)(58, 74.283784)(59, 74.310811)(60, 74.391892)(61, 74.378378)(62, 74.351351)(63, 74.378378)(64, 74.459459)(65, 74.445946)(66, 74.418919)(67, 74.418919)(68, 74.459459)(69, 74.486486)(70, 74.472973)(71, 74.486486)(72, 74.459459)(73, 74.513514)(74, 74.500000)(75, 74.540541)(76, 74.500000)(77, 74.500000)(78, 74.459459)(79, 74.432432)(80, 74.445946)(81, 74.486486)(82, 74.459459)(83, 74.513514)(84, 74.540541)(85, 74.527027)(86, 74.500000)(87, 74.500000)(88, 74.540541)(89, 74.540541)(90, 74.554054)(91, 74.540541)(92, 74.554054)(93, 74.594595)(94, 74.581081)(95, 74.540541)(96, 74.567568)(97, 74.554054)(98, 74.567568)(99, 74.581081)(100, 74.540541)
};
\legend{\safeopt, \algo}




\end{axis}

\end{tikzpicture}
    \caption{1 safety function safe region sizes}
  \end{subfigure}
  \hfill
  \begin{subfigure}[t]{0.3\textwidth}
    \begin{tikzpicture}

\begin{axis}[
tick label style={font=\footnotesize},
label style={font=\footnotesize},
legend style={nodes={scale=0.6, transform shape}},
legend pos=south east,
view={0}{90},
width=\figurewidth,
height=\figureheight/2,
xmin=40, xmax=100,
xtick={40, 60, 80, 100},
ymin=25, ymax=35,
ylabel={Safe region size},
major tick length=2pt,
axis lines*=left,
legend cell align=left,
clip=false]


\addplot [
color=cyan!60!black,
dotted,
line width=2pt,
]
coordinates{
(40, 26.833333)(41, 26.923077)(42, 26.987179)(43, 27.083333)(44, 27.121795)(45, 27.198718)(46, 27.320513)(47, 27.352564)(48, 27.391026)(49, 27.461538)(50, 27.570513)(51, 27.608974)(52, 27.673077)(53, 27.750000)(54, 27.820513)(55, 27.903846)(56, 27.974359)(57, 27.961538)(58, 28.006410)(59, 28.057692)(60, 28.102564)(61, 28.134615)(62, 28.205128)(63, 28.211538)(64, 28.256410)(65, 28.275641)(66, 28.352564)(67, 28.397436)(68, 28.461538)(69, 28.532051)(70, 28.570513)(71, 28.596154)(72, 28.621795)(73, 28.660256)(74, 28.711538)(75, 28.750000)(76, 28.801282)(77, 28.839744)(78, 28.891026)(79, 28.916667)(80, 28.942308)(81, 28.987179)(82, 29.032051)(83, 29.038462)(84, 29.044872)(85, 29.089744)(86, 29.102564)(87, 29.134615)(88, 29.224359)(89, 29.307692)(90, 29.320513)(91, 29.352564)(92, 29.397436)(93, 29.416667)(94, 29.442308)(95, 29.487179)(96, 29.506410)(97, 29.500000)(98, 29.512821)(99, 29.532051)(100, 29.544872)
};

\addplot [
color=orange!95!black,
solid,
line width=2pt,
]
coordinates{
(40, 31.685897)(41, 31.756410)(42, 31.801282)(43, 31.807692)(44, 31.852564)(45, 31.846154)(46, 31.865385)(47, 31.891026)(48, 31.897436)(49, 31.903846)(50, 31.929487)(51, 31.948718)(52, 31.948718)(53, 31.942308)(54, 31.967949)(55, 31.980769)(56, 31.967949)(57, 31.980769)(58, 31.967949)(59, 31.987179)(60, 31.980769)(61, 31.967949)(62, 31.980769)(63, 31.993590)(64, 31.987179)(65, 31.987179)(66, 32.012821)(67, 32.000000)(68, 32.019231)(69, 32.032051)(70, 32.032051)(71, 32.019231)(72, 32.051282)(73, 32.057692)(74, 32.051282)(75, 32.057692)(76, 32.064103)(77, 32.070513)(78, 32.057692)(79, 32.057692)(80, 32.064103)(81, 32.051282)(82, 32.057692)(83, 32.064103)(84, 32.070513)(85, 32.064103)(86, 32.064103)(87, 32.064103)(88, 32.064103)(89, 32.057692)(90, 32.051282)(91, 32.051282)(92, 32.057692)(93, 32.057692)(94, 32.064103)(95, 32.051282)(96, 32.051282)(97, 32.051282)(98, 32.051282)(99, 32.051282)(100, 32.051282)
};
\legend{\safeopt, \algo}




\end{axis}

\end{tikzpicture}
    \caption{3 safety functions safe region sizes}
  \end{subfigure}
  \hfill
  \begin{subfigure}[t]{0.3\textwidth}
    \begin{tikzpicture}

\begin{axis}[
tick label style={font=\footnotesize},
label style={font=\footnotesize},
legend style={nodes={scale=0.6, transform shape}},
legend pos=south east,
view={0}{90},
width=\figurewidth,
height=\figureheight/2,
xmin=40, xmax=100,
xtick={40, 60, 80, 100},
ymin=106, ymax=110,
ylabel={Safe region size},
major tick length=2pt,
axis lines*=left,
legend cell align=left,
clip=false]


\addplot [
color=cyan!60!black,
dotted,
line width=2pt,
]
coordinates{
(40, 106.248555)(41, 106.358382)(42, 106.514451)(43, 106.641618)(44, 106.838150)(45, 106.976879)(46, 107.075145)(47, 107.271676)(48, 107.421965)(49, 107.537572)(50, 107.578035)(51, 107.653179)(52, 107.780347)(53, 107.890173)(54, 108.000000)(55, 108.208092)(56, 108.283237)(57, 108.398844)(58, 108.514451)(59, 108.653179)(60, 108.734104)(61, 108.768786)(62, 108.797688)(63, 108.884393)(64, 109.075145)(65, 109.132948)(66, 109.173410)(67, 109.306358)(68, 109.352601)(69, 109.375723)(70, 109.404624)(71, 109.491329)(72, 109.502890)(73, 109.514451)(74, 109.595376)(75, 109.693642)(76, 109.716763)(77, 109.722543)(78, 109.745665)(79, 109.791908)(80, 109.849711)(81, 109.965318)(82, 110.028902)(83, 110.017341)(84, 110.040462)(85, 110.121387)(86, 110.202312)(87, 110.236994)(88, 110.271676)(89, 110.271676)(90, 110.294798)(91, 110.294798)(92, 110.335260)(93, 110.375723)(94, 110.416185)(95, 110.450867)(96, 110.485549)(97, 110.549133)(98, 110.589595)(99, 110.595376)(100, 110.601156)
};

\addplot [
color=orange!95!black,
solid,
line width=2pt,
]
coordinates{
(40, 107.537572)(41, 107.739884)(42, 107.786127)(43, 107.907514)(44, 107.994220)(45, 108.063584)(46, 108.213873)(47, 108.254335)(48, 108.352601)(49, 108.387283)(50, 108.473988)(51, 108.554913)(52, 108.612717)(53, 108.710983)(54, 108.774566)(55, 108.815029)(56, 108.901734)(57, 108.907514)(58, 108.988439)(59, 109.005780)(60, 109.080925)(61, 109.190751)(62, 109.167630)(63, 109.213873)(64, 109.242775)(65, 109.271676)(66, 109.306358)(67, 109.341040)(68, 109.404624)(69, 109.329480)(70, 109.404624)(71, 109.456647)(72, 109.468208)(73, 109.479769)(74, 109.497110)(75, 109.514451)(76, 109.566474)(77, 109.560694)(78, 109.549133)(79, 109.583815)(80, 109.583815)(81, 109.630058)(82, 109.739884)(83, 109.809249)(84, 109.826590)(85, 109.884393)(86, 109.901734)(87, 109.890173)(88, 109.872832)(89, 109.919075)(90, 109.942197)(91, 109.942197)(92, 109.936416)(93, 109.924855)(94, 109.930636)(95, 109.942197)(96, 109.947977)(97, 109.947977)(98, 109.971098)(99, 109.982659)(100, 109.988439)
};
\legend{\safeopt, \algo}




\end{axis}

\end{tikzpicture}
    \caption{1 safety function, dueling feedback safe region sizes}
  \end{subfigure}
  \caption{Results on three synthetic scenarios. The first row corresponds to the reward and the second row to the growth of the safe region sizes (higher is better for both). In both of these metrics, \algo performs at least as well as \safeopt. For clarity, we omit the first 40 iterations for each setting since the algorithms similarly expand the safe region during that phase.}
  \label{fig:results}
\end{figure*}
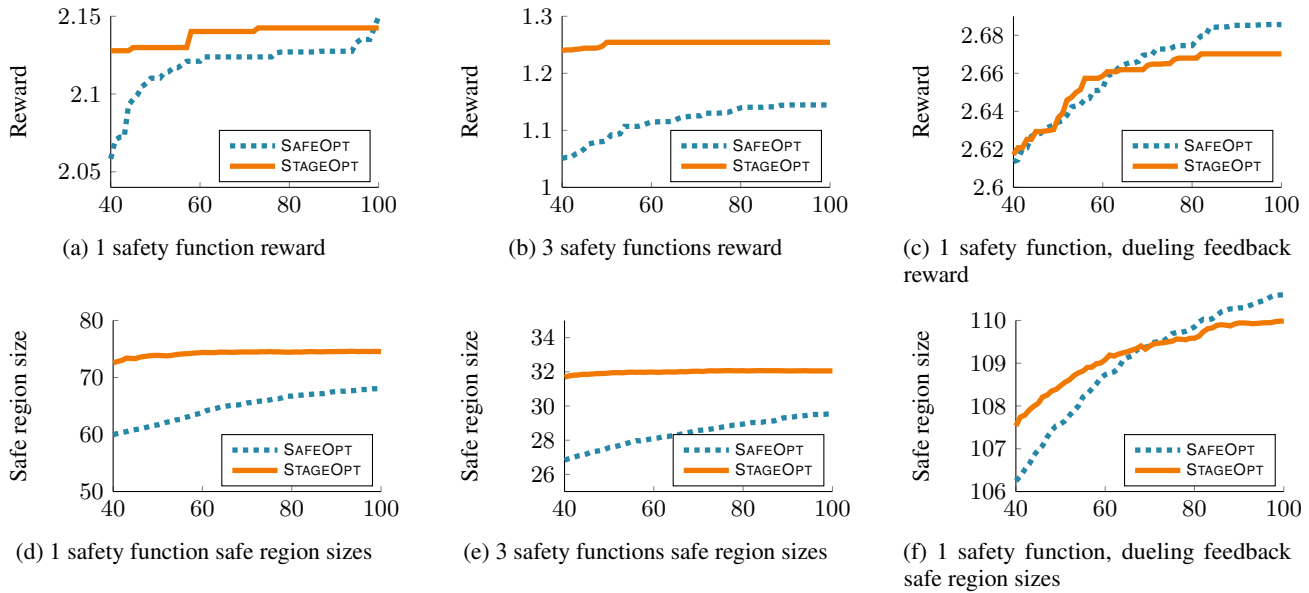

We evaluated our algorithm on synthetic data as well as on a live clinical experiment on spinal cord therapy. 

\paragraph{Modified \algo and \safeopt.}
In real applications, it may be difficult to compute an accurate estimate of the Lipschitz constants, which may have an adverse effect on the definition of the safe region and its expansion dynamics. In these scenarios, one can use a modified version of \safeopt that defines safe points using only the GPs \citep{berkenkamp16safe}. This modification can be directly applied to \algo as well; for clarity, we state the details here. Under this alternative definition, a point is classified as safe if the lower confidence bound of each of its safety GPs lies above the respective threshold:
$$S_t = \bigcap_{i}{\bigsetdef{\*x \in D}{\ell_t^i(\*x) \geq h_i}}.$$
A safe point is then an expander if an optimistic noiseless measurement of its upper confidence bound results in a non-safe point having all of its lower confidence bounds above the respective thresholds:
$$e_t(\*x) \defeq \Big|\bigcap_i \bigsetdef{\*x' \in D \setminus S_t}{\ell_{t, u_t(\*x)}(\*x) \geq h_i}\Big|.$$


\subsection{Synthetic Data}

We evaluated on several synthetic settings with various types of safety constraints and feedback. In each setting, the utility function was sampled from a zero-mean GP with Mat\'ern kernel ($\nu=1.2$) over the space $D=[0,1]^2$ uniformly discretized into $25 \times 25$ points. We considered the following safety constraint settings:
(i) One safety function $g_1$ sampled from a zero-mean GP with a Mat\'ern kernel with $\nu=1.2$.
(ii) Three safety functions $g_1, g_2, g_3$, sampled from zero-mean GPs with length scales 0.2, 0.4 and 0.8.

We set the amplitudes of the safety functions to be 0.1 that of the utility function, and the safety threshold for each safety function $g_i$ to be $\mu_i + \frac{1}{2}\sigma_i$. We define a point $x$ to be a \emph{safe seed} if it satisfies $g_i(x) > \mu_i + \sigma_i$.

We also considered several cases for feedback. For both safety settings, we examined the standard Gaussian noise-perturbed case, with $\sigma^2=0.0025$. We also ran experiments for the dueling feedback case and the first safety setting.

\textbf{Algorithms.} As discussed previously, \safeopt is the only other known algorithm that has similar guarantees in our setting, and serves as the main competitor to \algo. In addition, we also compared against the constrained Expected Improvement (CEI) algorithm from \citet{gelbart2014bayesian}. Since CEI only guarantees stepwise safety as opposed to over the entire time horizon, we set the safety threshold to be $\delta / T$ with $\delta = 0.1$ in order to match our setting. Naturally, with such a stringent threshold, CEI is not very competitive compared to \algo and \safeopt, as seen in Figure~\ref{fig:cei_compare}. In order to adequately distinguish between the latter two algorithms, we omit constrained EI results from all further figures.

\textbf{Results.} In each setting, we randomly sampled 30 combinations of utility and safety functions and ran \algo and \safeopt for $T=100$ iterations starting from each of 10 randomly sampled safe seeds. For \algo, we used a dynamic stopping criterion for the safe expansion phase (i.e. $T_0$) of when the safe region plateaus for 10 iterations, hard capped at 80 iterations. In these experiments, we primarily used GP-UCB in the utility optimization phase. We also tried two other common acquisition functions,  Expected Improvement (EI) and Maximum Probability of Improvement (MPI). However, we observed similar behavior between all three acquisition functions, since our algorithm quickly identifies the reachable safe region in most scenarios.

In Figure \ref{fig:results}, for each setting and algorithm, we plot both the growth of the size of the safe region as well as a notion of reward $r_t = \max_{1\leq i \leq t} f(\*x_i)$. Although there is some similarity between the performances of the algorithms, it is evident that \algo grows the safe region at least as fast as \safeopt, while also reaching a optimal sample point more quickly.

\subsection{Clinical Experiments}
We finally applied \algo to safely optimize clinical spinal cord stimulation in order to help tetraplegic patients regain physical mobility. The goal is to find effective stimulation therapies for patients with severe spinal cord injuries without introducing undesirable side effects. For example, bad stimulations could have negative effects on the rehabilitation and are often painful. This application is easily framed under our problem setting; the chosen configurations must stay above a safety threshold.

A total of 564 therapeutic trials were done with a tetraplegic patient in gripping experiments over 10 weeks. In each trial, one stimulating pattern was generated by the 32-channel-electrode, and was fixed within each trial. For a fixed electrode configuration, the stimulation frequency and amplitude were modulated synergistically in order to find those best for effective gripping. A similar setup was studied in \cite{sui2017correlational}. We optimized the electrode patterns with preference-based \algo (see Appendix \ref{app:duel}) and performed exhaustive search for stimulation frequency and amplitude over a narrow range.

\textbf{Results.} Figure~\ref{fig:scs_region} shows the reachable stimulating patterns by the algorithm under safety constraints. The physicians are confident that the total safe region has been reached between 300 and 400 iterations. In our experiments, \algo does not sample any unsafe stimulating patterns.


\begin{figure}[t]
\centering
  \setlength\figureheight{0.45\textwidth}
  \setlength\figurewidth{0.45\textwidth}
\input{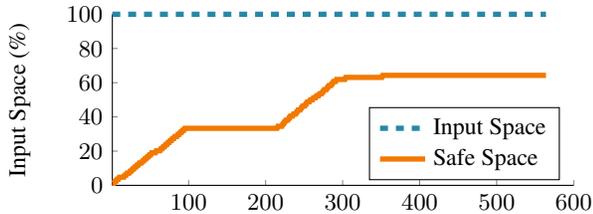}
\caption{Expansion of the safe region for spinal cord injury therapy. The orange solid line represents the growth of safe region over time, and the blue dashed line the total size of the input space.}
\label{fig:scs_region}
\end{figure}

\begin{figure}[t]
\centering
  \setlength\figureheight{0.45\textwidth}
  \setlength\figurewidth{0.45\textwidth}
\input{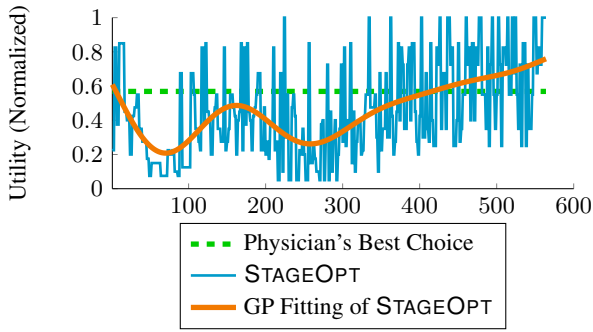}
\caption{Utilities within the safe region (larger is better). The green dashed line denotes the physician's best choice. The thin blue line shows the utilities of \algo at each iteration, and the orange solid line is a GP curve fitting of these utilities.}
\label{fig:scs_utility}
\end{figure}

Figure~\ref{fig:scs_utility} plots the utility measure of the stimulating pattern at each iteration. The orange solid line is a GP curve fitting of \algo (in thin blue). It clearly exceeds the physician's best choice (dotted green line) after around 400~iterations of online experiments.  These results demonstrate the practicality of \algo to safely optimize in challenging settings, such as those involving live human subjects.


\section{Conclusion \& Discussion}
\label{conclusion}

In this paper, we study the problem of safe Bayesian optimization, which is well suited to any setting requiring safe online optimization such as medical therapies, safe recommender systems, and safe robotic control. We proposed a general framework, \algo, which is able to tackle non-comparable safety constraints and utility function. We provide strong theoretical guarantees for \algo with safety functions and utility function sampled from Gaussian processes. Specifically, we bound the sample complexity to achieve an $\epsilon$-safe region and $\zeta$-optimal utility value within the safe region. The whole sampling process is guaranteed to be safe with high probability.

We compared \algo with classical Bayesian optimization methods and state-of-the-art safe optimization algorithms. We evaluated multiple cases such as single safety function, multiple safety functions, real-valued utility, and dueling-feedback utility. Our extensive experiments on synthetic data show that \algo can achieve its theoretical guarantees on safety and optimality. Its performance on safe expansion is among the best and utility maximization outperforms the state-of-the-art. 

This result also provides an efficient tool for online optimization in safety-critical applications. 
For instance, we applied \algo with dueling-feedback utility function on the gripping rehabilitation therapy for tetraplegic patients. Our clinical experiments demonstrated good performance in real human experiments of neural stimulation. The therapies proposed by \algo outperform the ones suggested by experienced physicians.

There are many interesting directions for future work.  For instance, we assume a static environment that does not evolve in response to the actions taken.  In our clinical application, this implies assuming that the patients' condition and response to stimuli do not improve over time.  Moving forward, it would be interesting to incorporate dynamics into our setting, which would lead to the multi-criteria safe reinforcement learning setting \cite{moldovan12safe,turchetta16safemdp,wachi2018safe}.

Another interesting direction is developing theoretically rigorous approaches outside of using Gaussian processes (GPs).  Although highly flexible, GPs require a well-specified prior and kernel in order to be effective. While one could use uniformed priors to model most settings, such priors tend to lead to very slow convergence.  One alternative is to automatically learn a good kernel \cite{wilson2016deep}.  Another approach is to assume a low-dimensional manifold within the high-dimensional uniformed kernel \cite{djolonga2013high}, which could also speed up learning.

\subsection*{Acknowledgements}
\vspace{-0.1in}
This research
was also supported in part by NSF Awards \#1564330 \&
\#1637598, JPL PDF IAMS100224, a Bloomberg Data Science
Research Grant, and a gift from Northrop Grumman.


\begin{small}
\bibliography{safe_bayesian}
\bibliographystyle{icml2018}
\end{small}



\newpage
\appendix
\section{Proofs}
\label{app:b}

\paragraph{Note} Without specific notification, the following lemmas and corollaries holds within each stage and for all $f_i$'s. All following lemmas and corollaries hold for any $\varnothing \subsetneq S_0 \subseteq D$, $h\in\mathbb{R}$, $\delta \in (0, 1)$, and $\epsilon > 0$.

\begin{lemma} \label{lem:confidence}
Assume $\|f\|^2_k \leq B$, and suppose the observation noise $n_t$ is a $\sigma$-sub-Gaussian stochastic process. Choose $\beta_t = B + \sigma \sqrt{2(\gamma_{t-1} + 1 +\log(1/\delta))}$ as shown in Equation~\ref{eq:beta}. Then for all $\delta \in (0,1)$, with probability at least $1-\delta$, for all iterations $t$ during the expansion stage of \algo, and for all $\*x\in D$, it holds that $f(\*x)\in C_t(\*x)$.
\end{lemma}
\begin{proof}
This lemma directly follows the Theorem~2 of \citet{chowdhury2017kernelized}. $C_t(\*x)$ represents the confidence interval which we construct in Section~\ref{sec:problem}.
\end{proof}

\begin{lemma}
  \label{lem:basics}
  For any $ t \geq 1$, the following properties hold:
  \begin{enumerate}[label=(\roman*)]
    \item $\forall \*x \in D, \ell_{t+1}(\*x) \geq \ell_t(\*x)$,
    \item $\forall \*x \in D, u_{t+1}(\*x) \leq u_t(\*x)$,
    \item $\forall \*x \in D, w_{t+1}(\*x) \leq w_t(\*x)$,
    \item $S_{t+1} \supseteq S_t \supseteq S_0$,
    \item $S \subseteq R \Rightarrow \Reps(S) \subseteq \Reps(R)$,
    \item $S \subseteq R \Rightarrow \Rbeps(S) \subseteq \Rbeps(R)$.
  \end{enumerate}
\end{lemma}
\begin{proof}
  (i), (ii), and (iii) follow directly from their definitions and the definition of $C_t(\*x)$.
  \begin{enumerate}[label=(\roman*)]
    \setcounter{enumi}{3}
    \item Proof by induction.
      For the base case, let $\*x \in S_0$.
      Then,
        \begin{align*}
          \ell_1(\*x) - L d(\*x, \*x) = \ell_1(\*x) \geq \ell_0(\*x) \geq h,
        \end{align*}
      where the last inequality follows from the initialization in Line~2 of Algorithm~\ref{alg:safe}.
      But then, from the above equation and Line~10 of Algorithm~\ref{alg:safe}, it follows that $\*x \in S_1$.
      
      For the induction step, assume that for some $t \geq 2$, $S_{t-1} \subseteq S_t$ and let $\*x \in S_t$.
      By Line~10 of Algorithm~\ref{alg:safe}, this means that $\exists \*z \in S_{t-1}, \ell_t(\*z) - L d(\*z, \*x) \geq h$.
      But, since $S_{t-1} \subseteq S_t$, it means that $\*z \in S_t$.
      Furthermore, by part (ii), $\ell_{t+1}(\*z) \geq \ell_t(\*z)$.
      Therefore, we conclude that $\ell_{t+1}(\*z) - L d(\*z, \*x) \geq h$, which implies that $\*x \in S_{t+1}$.
      \item Let $\*x \in \Reps(S)$. Then, by definition, $\exists \*z \in S, f(\*z) - L d(\*z, \*x) \geq h$.
        But, since $S \subseteq R$, it means that $\*z \in R$, and, therefore, $f(\*z) - L d(\*z, \*x) \geq h$ also implies that $\*x \in \Reps(R)$.
      \item This follows directly by repeatedly applying the result of part (v).
  \end{enumerate}
\end{proof}

\begin{lemma}
\label{lem:rkhs}
  Assume that $\|f\|^2_k \leq B$, $n_t$ is $\sigma$-sub-Gaussian, $\forall t\geq 1$.
  If $\beta_t = B + \sigma \sqrt{2(\gamma_{t-1} + 1 +\log(1/\delta))}$, then the following holds with probability at least $1-\delta$:
  \begin{align*}
  \forall t \geq 1\,\forall \*x \in D,\ |f(\*x) - \mu_{t-1}(\*x)| \leq \beta_t\sigma_{t-1}(\*x).
  \end{align*}
\end{lemma}
\begin{proof}
  See Theorem 2 by \citet{chowdhury2017kernelized}.
\end{proof}

\begin{cor}
  \label{cor:rkhs}
  For $\beta_t$ as above, the following holds with probability at least $1-\delta$:
  \begin{align*}
  \forall t \geq 1\,\forall \*x \in D,\ f(\*x) \in C_t(\*x).
  \end{align*}
  where $C_t(\*x)$ is the confidence interval at $\*x$ at $t$ iteration.
\end{cor}

In the following lemmas, we implicitly assume that the assumptions of Lemma~\ref{lem:rkhs} hold, and that $\beta_t$ is defined as above.

\begin{lemma}
  \label{lem:gtmt}
  For any $t_1 \geq t_0 \geq 1$, if $S_{t_1} = S_{t_0}$, then, for any $t$, such that $t_0 \leq t < t_1$, it holds that
  \begin{align*}
    G_{t+1} \subseteq G_t.
  \end{align*}
\end{lemma}
\begin{proof}
  Given the assumption that $S_t$ does not change, $G_{t+1} \subseteq G_t$ follows directly from the definitions of $G_t$.
  In particular, for $G_t$, note that for any $\*x \in S_t$, $g_t(\*x)$ is decreasing in $t$, since $u_t(\*x)$ is decreasing in $t$.
\end{proof}

\begin{lemma}
  \label{lem:wt}
  For any $t_1 \geq t_0 \geq 1$, if $S_{t_1} = S_{t_0}$ and $C_1 \defeq 8 / \log(1 + \sigma^{-2})$, then, for any $t$, such that $t_0 \leq t \leq t_1$, it holds that
    \begin{align*}
      w_t(\*x_t) \leq \sqrt{\frac{C_1 \beta_t^2 \gamma_t}{t-t_0}}.
    \end{align*}
\end{lemma}
\begin{proof}
Based on the results of Lemma~\ref{lem:gtmt}, the definition of $\*x_t \defeq \argmax_{\*x \in G_t}(w_t(\*x))$, and the fact that, by definition, $w_t(\*x_t) \leq 2\beta_t\sigma_{t-1}(\*x_t)$, the proof is a straight forward analogy of Lemma 4 by \citet{chowdhury2017kernelized}.
$\gamma_t$ will be replaced by $\gamma_{nt}$ for multiple safety functions where $n$ is the number of safety functions. This directly follows the Theorem~1 in \citet{berkenkamp16bayesian}
\end{proof}

\begin{cor}
  \label{cor:wt}
  For any $t \geq 1$, if $C_1$ is defined as above, $\T$ is the smallest positive integer satisfying $\displaystyle\frac{\T}{\beta_{t+\T}^2 \gamma_{t+\T}} \geq \frac{C_1}{\epsilon^2}$, and $S_{t+\T} = S_t$, then, for any $\*x \in G_{t+\T}$, it holds that
  \begin{align*}
    w_{t+\T}(\*x) \leq \epsilon.
  \end{align*}
\end{cor}

In the following lemmas, we assume that $C_1$ and $\T$ are defined as above.

\begin{lemma}
  \label{lem:expansion0}
  For any $t \geq 1$, if $\Rbeps(S_0) \setminus S_t \not= \varnothing$, then $\Reps(S_t) \setminus S_t \not= \varnothing$.
\end{lemma}
\begin{proof}
  Assume, to the contrary, that $\Reps(S_t) \setminus S_t = \varnothing$.
  By definition, $\Reps(S_t) \supseteq S_t$, therefore $\Reps(S_t) = S_t$.
  Iteratively applying $\Reps$ to both sides, we get in the limit $\Rbeps(S_t) = S_t$.
  But then, by Lemma~\ref{lem:basics} (iv) and (vi), we get
  \begin{align}
    \label{eq:expansion0}
    \Rbeps(S_0) \subseteq \Rbeps(S_t) = S_t,
  \end{align}
  which contradicts the lemma's assumption that $\Rbeps(S_0) \setminus S_t \not= \varnothing$.
\end{proof}

\begin{lemma}
  \label{lem:expansion}
  Within the expansion stage, for any $t \geq 1$, if $\Rbeps(S_0) \setminus S_t \not= \varnothing$, then the following holds with probability at least $1-\delta$:
  \begin{align*}
    S_{t+\T} \supsetneq S_t.
  \end{align*}
\end{lemma}
\begin{proof}
  By Lemma~\ref{lem:expansion0}, we get that, $\Reps(S_t) \setminus S_t \not= \varnothing$, For equivalently, by definition,
    \begin{align}
      \label{eq:expansion1}
      \exists \*x \in \Reps(S_t) \setminus S_t\,\exists \*z \in S_t, f(\*z) - \epsilon - L d(\*z, \*x) \geq h.
    \end{align}
    
  Now, assume, to the contrary, that $S_{t+\T} = S_t$ (see Lemma~\ref{lem:basics} (iv)), which implies that $\*x \in D \setminus S_{t+\T}$ and $\*z \in S_{t+\T}$.
  Then, we have
  \begin{align*}
    u_{t+\T}(\*z) - L d(\*z, \*x) &\geq f(\*z) - L d(\*z, \*x) \tag*{by Lemma~\ref{lem:rkhs}}\\
    &\geq f(\*z) - \epsilon - L d(\*z, \*x) \\
    &\geq h. \tag*{by Equation~\ref{eq:expansion1}}
  \end{align*}
  Therefore, by definition, $g_{t+\T}(\*z) > 0$, which implies $\*z \in G_{t+\T}$.
  
  Finally, since $S_{t+\T} = S_t$ and $\*z \in G_{t+\T}$, we can use Corollary~\ref{cor:wt} as follows:
  \begin{align*}
    \ell_{t+\T}(\*z) - L d(\*z, \*x) &\geq \ell_{t+\T} - f(\*z) + \epsilon + h \tag*{by Equation~\ref{eq:expansion1}}\\
    &\geq -w_{t+\T}(\*z) + \epsilon + h \tag*{by Lemma~\ref{lem:rkhs}}\\
    &\geq h \tag*{by Corollary~\ref{cor:wt}}.
  \end{align*}
  This means that by Line 10 of Algorithm~\ref{alg:safe} we get $\*x \in S_{t+\T}$, which is a contradiction.
\end{proof}

\begin{lemma}
  \label{lem:stleq}
  For any $t \geq 0$, the following holds with probability at least $1-\delta$:
  \begin{align*}
    S_t \subseteq \Rbo(S_0).
  \end{align*}
\end{lemma}
\begin{proof}
  Proof by induction. For the base case, $t = 0$, we have by definition that $S_0 \subseteq \Rbo(S_0)$.
      
  For the induction step, assume that for some $t \geq 1$, $S_{t-1} \subseteq \Rbo(S_0)$.
  Let $\*x \in S_t$, which, by definition, means $\exists \*z \in S_{t-1}$, such that
  \begin{align*}
    & \ell_t(\*z) - L d(\*z, \*x) \geq h\\
    \Rightarrow\ \ & f(\*z) - L d(\*z, \*x) \geq h. \tag*{by Lemma~\ref{lem:rkhs}}\\
  \end{align*}
  Then, by definition of $\Rbo$ and the fact that $\*z \in \Rbo(S_0)$, it follows that $\*x \in \Rbo(S_0)$.
\end{proof}

\begin{lemma}
  \label{lem:existence}
  Let $t^*$ be the smallest integer, such that $t^* \geq |\Rbo(S_0)|T_{t^*}$.
  Then, there exists $t_0 \leq t^*$, such that $S_{t_0+T_{t_0}} = S_{t_0}$.
\end{lemma}
\begin{proof}
  Assume, to the contrary, that for any $t \leq t^*$, $S_t \subsetneq S_{t+T_t}$. (By Lemma~\ref{lem:basics} (iv), we know that $S_t \subseteq S_{t+T_t}$.)
  Since $\T$ is increasing in $t$, we have
  \begin{align*}
    S_0 \subsetneq S_{T_0} \subseteq S_{T_{t^*}} \subsetneq S_{T_{t^*}+T_{T_{t^*}}} \subseteq S_{2T_{t^*}} \subsetneq \cdots,
  \end{align*}
  which implies that, for any $0 \leq k \leq |\Rbo(S_0)|$, it holds that $|S_{kT_{t^*}}| > k$.
  In particular, for $k^* \defeq |\Rbo(S_0)|$, we get
  \begin{align*}
    |S_{k^*T}| > |\Rbo(S_0)|
  \end{align*}
  which contradicts $S_{k^*T} \subseteq \Rbo(S_0)$ by Lemma~\ref{lem:stleq}.
\end{proof}

\begin{cor}
  \label{cor:existence}
  Within the expansion stage, let $t^*$ be the smallest integer, such that $\displaystyle\frac{t^*}{\beta_{t^*}^2\gamma_{t^*}} \geq \frac{C_1|\Rbo(S_0)|}{\epsilon^2}$.
  Then, there exists $t_0 \leq t^*$, such that $S_{t_0+T_{t_0}} = S_{t_0}$.
\end{cor}
\begin{proof}
This is a direct consequence of combining Lemma~\ref{lem:existence} and Corollary~\ref{cor:wt}. $\gamma_{t^*}$ can be replaced by $\gamma_{nt^*}$ for $n$ safety functions.
\end{proof}

\begin{lemma}
  \label{lem:safe}
  If $f$ is \LLC, then, for any $t \geq 0$, the following holds with probability at least $1-\delta$:
  \begin{align*}
    \forall \*x \in S_t, f(\*x) \geq h.
  \end{align*}
\end{lemma}
\begin{proof}
  We will prove this by induction.
  For the base case $t = 0$, by definition, for any $\*x \in S_0$, $f(\*x) \geq h$.
  
  For the induction step, assume that for some $t \geq 1$, for any $\*x \in S_{t-1}$, $f(\*x) \geq h$.
  Then, for any $\*x \in S_t$, by definition, $\exists \*z \in S_{t-1}$,
  \begin{align*}
    h &\leq \ell_t(\*z) - L d(\*z, \*x)\\
      &\leq f(\*z) - L d(\*z, \*x) \tag*{by Lemma~\ref{lem:rkhs}}\\
      &\leq f(\*x) \tag*{by $L$-Lipschitz-continuity}.
  \end{align*}
\end{proof}

\begin{proof}[Proof of Theorem~\ref{thm:safe}]
  The first part of the theorem is a direct consequence of Lemma~\ref{lem:safe}.
  The second part follows from the combination of Lemma~\ref{lem:expansion} and Lemma~\ref{lem:existence}.
\end{proof}

We then consider the optimization stage of \algo.
\begin{lemma}
  \label{lem:opt}
  Suppose utility function $f$ satisfies $\|f\|^2_k \leq B$, $\delta \in (0,1)$, and noise $n_t$ is $\sigma$-sub-Gaussian. $\beta_t = B + \sigma \sqrt{2(\gamma_{t-1} + 1 +\log(1/\delta))}$ with utility function parameters.  $Y$ the time horizon for optimization stage. Suppose we run the optimization stage of \algo within the expansion stage safe region $\Reps^{T_0}(S_0)$.
Then with probability at least $1-\delta$, the average regret $\bar{r}_Y$ satisfies
$$\bar{r}_Y = \frac{R_Y}{Y} \leq \frac{4\sqrt{2}}{\sqrt{Y}}(B\sqrt{\gamma_Y} + \sigma \sqrt{2\gamma_Y(\gamma_Y + 1 +\log(1/\delta))}).$$
\end{lemma}

\begin{proof}
After $Y$ iterations in the optimization stage, we have 
\begin{align*}
\label{eq:lem4}
\Sigma_{t=1}^Y \sigma_{t-1}(\*x_t) \leq \sqrt{4(Y+2)\gamma_Y}
\end{align*}
following the Lemma 4 of \citet{chowdhury2017kernelized}.

By definition, stepwise regret
\begin{align*}
r_t &= f(\*x^*) - f(\*x_t) \\ 
&\leq \mu_{t-1}(\*x_t) + \beta_t \sigma_{t-1}(\*x_t) - f(\*x_t) \\ 
&\leq 2\beta_t\sigma_{t-1}(\*x_t).
\end{align*}
Substitute $r_t$ and results of Lemma~\ref{lem:opt} into the total regret $R_Y$ then we have
\begin{align*}
R_Y &= \Sigma_{t=1}^Y r_t \\ 
&\leq 2\beta_Y\Sigma_{t=1}^Y \sigma_{t-1}(\*x_t) \\
&\leq 2\beta_Y\sqrt{4(Y+2)\gamma_Y} \\
&\leq 4(B+\sigma\sqrt{2(\gamma_Y+1+\ln(1/\delta))})\sqrt{(Y+2)\gamma_Y} \\
&\leq 4(B+\sigma\sqrt{2(\gamma_Y+1+\ln(1/\delta))})\sqrt{2Y\gamma_Y}
\end{align*}
\end{proof}

\begin{proof}[Proof of Theorem~\ref{thm:opt}]
  Given $Y$ be the smallest positive integer satisfying 
$$\frac{4\sqrt{2}}{\sqrt{Y}}(B\sqrt{\gamma_Y} + \sigma \sqrt{2\gamma_Y(\gamma_Y + 1 +\log(1/\delta))}) \leq \zeta$$
From Lemma~\ref{lem:opt}, we immediately have $\bar{r}_Y = R_Y/Y \leq \zeta$. Then $\exists \ \ \hat{\*x}^*$ in the samples such that $f(\hat{\*x}^*) \geq f(\*x^*) - \zeta$.
\end{proof}

\section{\algo with dueling feedback}
\label{app:duel}

For completeness, we present the pseudocode for \algo under the dueling feedback described in section \ref{sec:related}. In this setting, for the utility function the algorithm receives Bernoulli feedback according to some link function $\phi$ between the current sample point $\*x_t$ and the previous sample point $\*x_{t-1}$. The safety GPs receive real-valued feedback in order to preserve the safety guarantees. We note that this formulation is very similar to that of \kersparring \citep{sui2017multi} but in which only one point is selected at each iteration. If one can sample multiple points at each iteration, \kersparring can be used in place of GP-UCB.

\begin{algorithm}[t]
  \caption{\algo with dueling feedback}
  \label{alg:safe_duel}
{\small
\begin{algorithmic}[1]
\STATE {\bfseries Input:} sample set $D$, $i\in\{1,\dots,n\}$,\\  
           \hspace{3.0em}GP prior for utility function $f$,\\  
           \hspace{3.0em}GP priors for safety functions $g_i$,\\
           \hspace{3.0em}Lipschitz constants $L_i$ for $g_i$,\\
           \hspace{3.0em}safe seed set $S_0$,\\
           \hspace{3.0em}safety threshold $h_i$,\\
           \hspace{3.0em}accuracies $\epsilon$ (expansion), $\zeta$ (optimization),\\
           \hspace{3.0em}link function $\phi$.
  \LET{$C_0^i(\*x)$}{$[h_i, \infty)$, for all $\*x \in S_0$}
  \LET{$C_0^i(\*x)$}{$\mathbb{R}$, for all $\*x \in D \setminus S_0$}
  \LET{$Q_0^i(\*x)$}{$\mathbb{R}$, for all $\*x \in D$}
  \LET{$C_0^f(\*x)$}{$\mathbb{R}$, for all $\*x \in D$}
  \LET{$Q_0^f(\*x)$}{$\mathbb{R}$, for all $\*x \in D$}

  \FOR{$t = 1, \ldots\, T_0$}
    \LET{$C_t^i(\*x)$}{$C_{t-1}^i(\*x) \cap Q_{t-1}^i(\*x)$}
    \LET{$C_t^f(\*x)$}{$C_{t-1}^f(\*x) \cap Q_{t-1}^f(\*x)$}
    \LET{$S_t$}{$\bigcap_{i}{\bigcup_{\*x \in S_{t-1}}\bigsetdef{\*x' \in D}{\ell_t^i(\*x) - L_i d(\*x, \*x') \geq h_i}}$}
    \LET{$G_t$}{$\bigsetdef{\*x \in S_t}{e_t(\*x) > 0}$}
    \IF{$\forall i, \epsilon_t^i < \epsilon$}
    \LET{$\*x_t$}{$\argmax_{\*x \in G_t, i\in\{1,\dots,n\}}w_t^i(\*x)$} 
    \ELSE
    \LET{$\*x_t$}{$\argmax_{\*x\in S_t} \mu_{t-1}^f(\*x) + \beta_t\sigma_{t-1}^f(\*x)$}
    \ENDIF
    \LET{$y_{f, t}$}{$\text{Bernoulli}(\phi(f(\*x_t), f(\*x_{t-1})))$}
    \LET{$y_{i, t}$}{$g_i(\*x_t) + n_{i, t}$}
    \STATE Compute $Q_{f, t}(\*x)$ and $Q_{i, t}(\*x)$, for all $\*x \in S_t$
  \ENDFOR
  \FOR{$t=T_0+1, \ldots, T$}
    \LET{$C_t^f(\*x)$}{$C_{t-1}^f(\*x) \cap Q_{t-1}^f(\*x)$}
    \LET{$\*x_t$}{$\argmax_{\*x\in S_t} \mu_{t-1}^f(\*x) + \beta_t\sigma_{t-1}^f(\*x)$}
    \LET{$y_{f, t}$}{$\text{Bernoulli}(\phi(f(\*x_t), f(\*x_{t-1})))$}
    \LET{$y_{i, t}$}{$g_i(\*x_t) + n_{i, t}$}
    \STATE Compute $Q_{f, t}(\*x)$ and $Q_{i, t}(\*x)$, for all $\*x \in S_t$ 
  \ENDFOR
\end{algorithmic}
}
\end{algorithm}

\end{document}